\documentclass[conference]{IEEEtran}
\IEEEoverridecommandlockouts
\usepackage{amsmath, amssymb, amsfonts, amsthm}  
\DeclareMathOperator*{\argmin}{arg\,min}          
\theoremstyle{remark}
\newtheorem{remark}{Remark}

\theoremstyle{definition}
\newtheorem{definition}{Definition}
\newtheorem{example}{Example}

\theoremstyle{plain}
\newtheorem{theorem}{Theorem}

\usepackage{weiwAlgorithm} 

\usepackage{graphicx}              
\usepackage{subcaption}            
\usepackage{multirow}              
\usepackage{textcomp}              
\usepackage{url}                   
\usepackage{comment}               

\usepackage{enumerate}             
\usepackage{enumitem}
\usepackage{balance}

\usepackage{bbm}                   
\usepackage{pifont}                
\usepackage{utfsym}                
\usepackage{fontawesome}           
\usepackage{xcolor}                

\usepackage[normalem]{ulem}        

\newcommand{\revise}[1]{\textcolor{blue}{#1}}
\renewcommand{\revise}[1]{#1} 
\newcommand{\myblue}[1]{{\color{blue} #1}\xspace}
\renewcommand{\myblue}[1]{#1}
\newcommand{\red}[1]{{\color{red} #1}\xspace}
\renewcommand{\red}[1]{#1}
\usepackage{cite}                  
\usepackage[colorlinks,citecolor=black,linkcolor=black,urlcolor=black]{hyperref}

\usepackage{booktabs,tabularx,threeparttable}
\usepackage{xcolor}
\def\BibTeX{{\rm B\kern-.05em{\sc i\kern-.025em b}\kern-.08em
    T\kern-.1667em\lower.7ex\hbox{E}\kern-.125emX}}

\newcommand{\Cmt}[1]{
\vspace{0.6mm} \tcp*[h]{\underline{#1}}\\ \vspace{0.6mm}
}
\newcommand{\myparagraph}[1]{\vspace{0.5mm} \noindent \textbf{#1}.}
\newcommand{\myparagraphexp}[1]{\noindent \underline{#1.}}

\begin{document}

\title{C$^{2}$TC: A Training-Free Framework for Efficient Tabular Data Condensation}
\DeclareRobustCommand*{\IEEEauthorrefmark}[1]{%
  \raisebox{0pt}[0pt][0pt]{\textsuperscript{\footnotesize\ensuremath{#1}}}
}

\author{\IEEEauthorblockN{Sijia Xu\IEEEauthorrefmark{1}, Fan Li\IEEEauthorrefmark{1 \ast \thanks{\textsuperscript{$\ast$}Corresponding author}}, Xiaoyang Wang\IEEEauthorrefmark{1}, Zhengyi Yang\IEEEauthorrefmark{1},
and Xuemin Lin\IEEEauthorrefmark{2}}

\IEEEauthorblockA{
\IEEEauthorrefmark{1}University of New South Wales, Australia 
\IEEEauthorrefmark{2}Shanghai Jiao Tong University, China \\
\{sijia.xu, fan.li8, xiaoyang.wang1, zhengyi.yang\}@unsw.edu.au, xuemin.lin@sjtu.edu.cn
}
}


\maketitle


\begin{abstract}
\revise{
Tabular data is the primary data format in industrial relational databases, underpinning modern data analytics and decision-making.
However, the increasing scale of tabular data poses significant computational and storage challenges to learning-based analytical systems. This highlights the need for data-efficient learning, which enables effective model training and generalization using substantially fewer samples. 
Dataset condensation (DC) has emerged as a promising data-centric paradigm that synthesizes small yet informative datasets to preserve data utility while reducing storage and training costs. 
However, existing DC methods are computationally intensive due to reliance on complex gradient-based optimization. Moreover, they often overlook key characteristics of tabular data, such as heterogeneous features and class imbalance.
To address these limitations, we introduce $\text{C}^{2}\text{TC}$ (\underline{C}lass-Adaptive \underline{C}lustering for \underline{T}abular \underline{C}ondensation), the first training-free tabular dataset condensation framework that jointly optimizes class allocation and feature representation, enabling efficient and scalable condensation.
Specifically, we reformulate the dataset condensation objective into a novel class-adaptive cluster allocation problem (CCAP), which eliminates costly training and integrates adaptive label allocation to handle class imbalance.
To solve the NP-hard CCAP, we develop HFILS, a heuristic local search that alternates between soft allocation and class-wise clustering to efficiently obtain high-quality solutions.
Moreover, a hybrid categorical feature encoding (HCFE) is proposed for semantics-preserving clustering of heterogeneous discrete attributes.
Extensive experiments on 10 real-world datasets demonstrate that $\text{C}^{2}\text{TC}$ improves efficiency by at least 2 orders of magnitude over state-of-the-art baselines, while achieving superior downstream performance.}
\end{abstract}

\maketitle



\section{Introduction}
\label{intro}

\revise{Tabular data, represented as rows and columns and comprising heterogeneous feature types (e.g., integer, floating point, and string), is the most widely used data format in relational databases and enterprise data warehouses~\cite{OLAP,li2025efficient,yang2024parallel,tang2025tabular}. It serves as the backbone of modern analytical workflows, increasingly integrated with machine learning models for prediction and automation, supporting a wide range of critical applications, such as financial risk assessment~\cite{MCDB-R,li2024adarisk}, recommendation systems~\cite{sarwat2013recdb,wang2026beyond}, and medical diagnosis~\cite{CohortNet}.
Recently, to better extract value from tabular data, deep tabular learning~\cite{deng2022turl,Observatory,Creating}, which eliminates the need for manual feature engineering and enables end-to-end representation learning, has received increasing attention as an effective paradigm for enhancing predictive performance in data-driven systems.
}

\revise{However, the increasing scale of tabular datasets imposes substantial computational and storage burdens on leaning-based data analytics systems, especially in resource-constrained scenarios.
To illustrate, training popular deep tabular models on the Epsilon dataset~\cite{GLMNET}, which contains only $5 \times 10^{5}$ samples and $2 \times 10^{3}$ features, suffers from severe efficiency bottleneck on a single NVIDIA A800 GPU: TabR~\cite{gorishniy2024tabr} requires 8 hours, while FT-Transformer~\cite{gorishniy2021revisiting} takes more than two days to complete. Beyond the computational expense, handling large-scale tables also introduces significant I/O overhead, further challenging the capacity and throughput of modern data management pipelines. The problem becomes even more pronounced in data-intensive scenarios, such as neural architecture search~\cite{li2021autood,elsken2019neural} and continual learning~\cite{de2021continual,chu2023continual}, where datasets must be repeatedly loaded, cached, and processed to train and update hundreds of models. Effectively managing large-scale tabular data while ensuring efficient and scalable learning has therefore become a critical challenge for modern data analytics systems. This raises a natural question: \textit{can we reduce the size of tabular datasets while preserving their utility for model training?}}

\revise{A straightforward and effective approach to achieve data reduction is through coreset selection~\cite{hadar2024datamap,camel,wang2022coresets}, which constructs a representative subset of full datasets while preserving their essential properties for downstream data processing and learning tasks.}
%
However, coreset selection methods are often guided by intermediate optimization objectives that may not align with downstream tasks, resulting in suboptimal performance and limited generalization. 
To overcome this limitation, 
\revise{dataset condensation (DC)~\cite{zhao2021dataset,cazenavette2022dataset,zhao2023dataset}, which aims to synthesize compact yet informative datasets that preserve the training utility of their large-scale counterparts, has emerged as an effective data-centric paradigm for efficient learning.}
DC methods typically optimize synthetic data by aligning surrogate objectives between the condensed and original datasets during model training, thereby improving transferability and preserving utility.
For instance, parameter-matching approaches align gradient directions~\cite{zhao2021dataset,lee2022dataset,jiang2023delving} or training trajectories~\cite{cazenavette2022dataset,guotowards,zhong2025towards}, whereas distribution-matching strategies~\cite{zhao2023dataset,zhao2023improved,malakshan2025decomposed} aim to match class-conditional feature embeddings between the original and the synthetic one. 
\myblue{Recently, for efficient graph condensation, \cite{gao2025rethinking} proposes a novel framework
that performs class partitioning on structure-aware representations.}

\noindent\textbf{Limitations.} 
\revise{While dataset condensation has shown effectiveness across various modalities such as images~\cite{zhao2023dataset,cazenavette2022dataset}, graphs~\cite{sun2024gc,yang2025stgcond,li2025tcgu}, and time series~\cite{miao2024less}, research on tabular data remains largely unexplored.
Applying existing DC techniques to tabular data suffers from three key limitations:}

\noindent \revise{\underline{\textit{L1: 
High computational cost}}: 
Existing condensation strategies are computationally intensive and poorly scalable due to their reliance on complex gradient-based objectives. Parameter matching methods require long-range model training with repeated gradient or parameter alignment, while distribution matching approaches, despite removing explicit model training, still involve costly gradient-based optimization in the embedding space. These limitations make them impractical for large-scale tables pervasive in industrial databases.}

\noindent  \underline{\textit{L2: Utility loss under class imbalance}}: Existing DC methods typically adopt two pre-defined label \revise{allocation} strategies: ratio-preserving allocation (Ratio),
\revise{which maintains the original label distribution in the condensed data,}
and fixed instances per class (FIPC), which assigns an equal and fixed number of samples to each class. 
However, real-world tabular datasets are often highly imbalanced, with minority classes at least an order of magnitude smaller than majority classes~\cite{liu2020self,he2009learning}.

\begin{example}
\revise{To illustrate the effect of such imbalance, in Figure~\ref{fig:class_alloc_comparison},} 
we condense two imbalanced datasets Adult (AD)~\cite{kohavi1996scaling} and Covertype (CO)\cite{blackard1999comparative} at a 0.1\% reduction ratio under both strategies using three representative DC methods, and evaluate utility by training an MLP on the condensed datasets. Detailed experimental setups can be found in Section~\ref{exp_setups}.
Our observations are twofold: $(i)$ Ratio yields higher overall predictive performance (accuracy) than FIPC, but with larger disparities in per-class performance (lower Macro-F1);
$(ii)$ Macro-F1 drops much more sharply than accuracy after condensation, indicating that dataset condensation exacerbates performance imbalance across classes. These results indicate that neither strategy achieves a satisfactory trade-off between global performance and class-wise fairness on imbalanced datasets, thus diminishing the utility of the condensed data.
\end{example}

\noindent  \underline{\textit{L3: Information loss in heterogeneous features}}: Tabular data is characterized by heterogeneous feature types, including numerical and categorical attributes. The latter often appear as boolean, integer, or string values and are crucial for representing discrete semantics. However, existing DC methods overlook the effective encoding and use of categorical features, leading to information loss and degraded condensation quality.

\begin{figure}[t]
  \centering
  \includegraphics[width=0.75\linewidth]{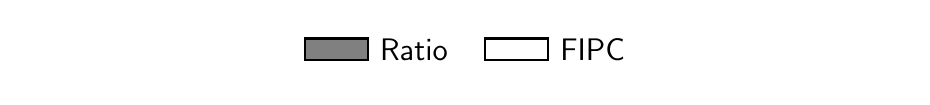}
  \begin{subfigure}[b]{0.48\columnwidth}
    \centering
    \includegraphics[width=0.92\linewidth]{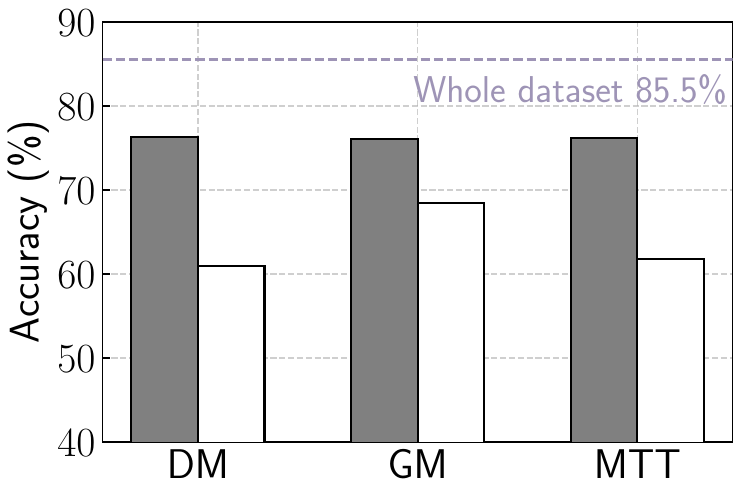}
    \caption{AD Accuracy}
  \end{subfigure}
  \hfill
  \begin{subfigure}[b]{0.48\columnwidth}
    \centering
    \includegraphics[width=0.92\linewidth]{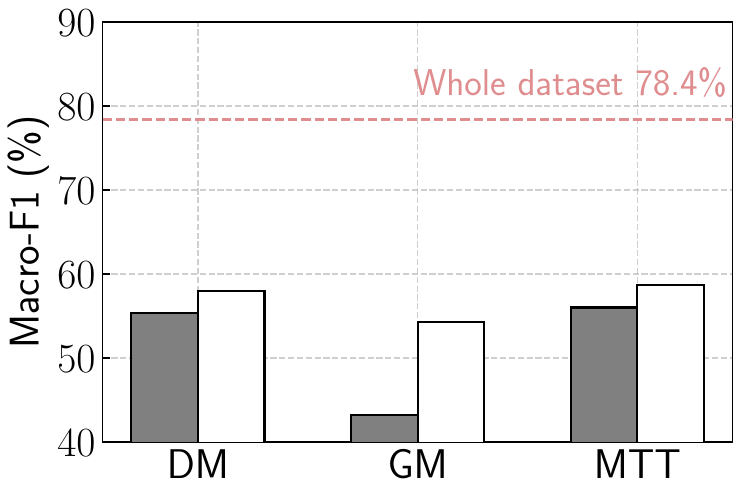}
    \caption{AD Macro-F1}
  \end{subfigure}


  \begin{subfigure}[b]{0.48\columnwidth}
    \centering
    \includegraphics[width=0.92\linewidth]{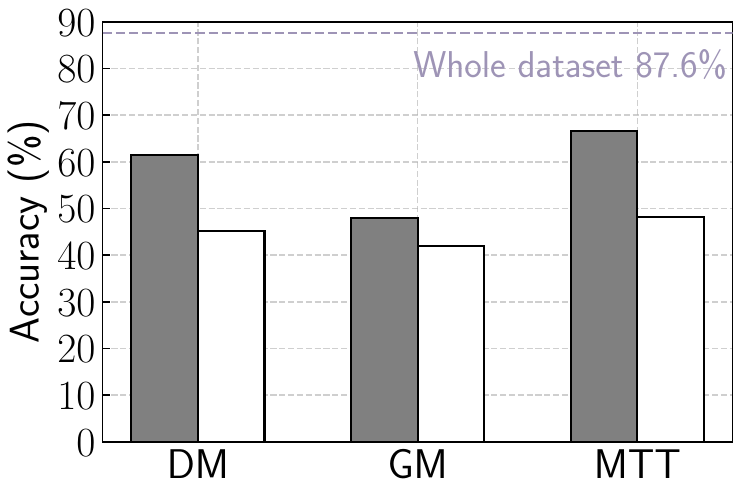}
    \caption{CO Accuracy}
  \end{subfigure}
  \hfill
  \begin{subfigure}[b]{0.48\columnwidth}
    \centering
    \includegraphics[width=0.92\linewidth]{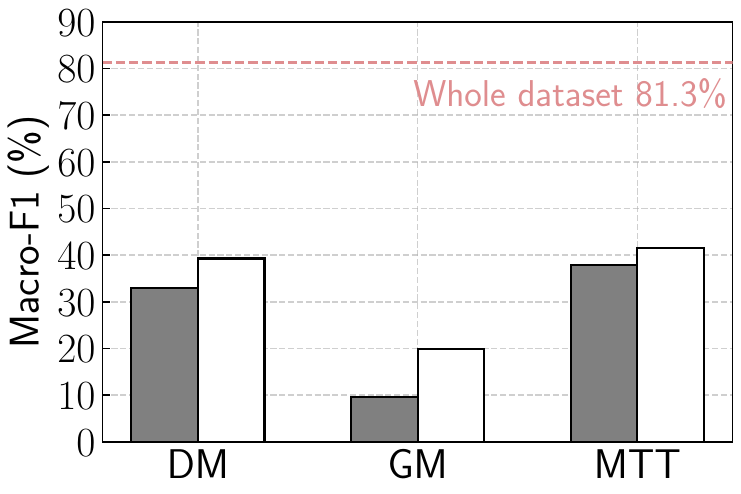}
    \caption{CO Macro-F1}
  \end{subfigure}

  \caption{Comparison of label \revise{allocation} strategies.}
  \label{fig:class_alloc_comparison}
\end{figure}

\noindent\textbf{Contributions.} 
\revise{ 
To address the above issues, we introduce $\text{C}^{2}\text{TC}$, the \textit{first} training-free and class-adaptive tabular data condensation framework. 
\myblue{Motivated by the recent analysis in~\cite{gao2025rethinking}},
we reformulate the optimization objective of data condensation to enable efficient and class-adaptive data reduction, which consists of three steps: $(i)$ We theoretically revisit existing DC methods and unify their objectives into a class-level matching paradigm that aligns each class’s prototype embeddings between the condensed and original datasets to preserve class-level representation consistency. $(ii)$ We extend this paradigm to a class-to-sample matching scheme for finer-grained condensation, which we reformulate as a class partitioning task in the raw feature space that can be efficiently solved using standard clustering algorithms (e.g., K-means \cite{kmeans++,macqueen1967multivariate}). This formulation removes the need for gradient-based training and significantly reduces computational overhead (\textit{L1}). $(iii)$ Based on this simplified objective, we incorporate dynamic label allocation to address class imbalance, yielding a novel class-adaptive cluster allocation problem (CCAP) that naturally supports label-adaptive condensation (\textit{L2}).
However, the CCAP is proven to be an NP-hard problem.
To efficiently solve this challenging problem, we design a heuristic first-improvement local search (HFILS) guided by the \textit{elbow method}~\cite{kodinariya2013review}. 
HFILS alternates between soft allocation and class-wise clustering, updating the current label allocation whenever the overall clustering loss decreases and thereby achieving fast convergence to high-quality feasible solutions.
Moreover, a hybrid categorical feature encoding (HCFE) is developed to transform heterogeneous discrete features into a unified numerical space, enabling semantic-preserving clustering during optimization (\textit{L3}). 
Extensive experiments on 10 real-world tabular datasets demonstrate the superior efficiency and effectiveness of $\text{C}^{2}\text{TC}$ compared with state-of-the-art DC baselines.
}
The main contributions of the paper are summarized as follows.

\begin{itemize}[leftmargin=*]
  \item 
  \revise{
    We propose $\text{C}^{2}\text{TC}$, the \textbf{first} training-free tabular data condensation framework with label-adaptive design, enabling efficient and scalable deep tabular learning workflows.
  }
  \item{
  \revise{
  We reformulate the complex gradient-based data condensation objectives into a \textbf{training-free} and \textbf{class-adaptive} combinatorial optimization problem (i.e., CCAP), which facilitates a more efficient and scalable condensation while supporting dynamically adaptive label allocation.
  }
  }
  \item{
  \revise{
  To solve the NP-hard CCAP, a heuristic first-improvement local search (HFILS) is proposed to efficiently search for high-quality label allocations.
  Moreover, we design a hybrid categorical feature encoding (HCFE) module to preserve semantic information in heterogeneous features, ensuring effective clustering.
  }
  }
  \item{
  \revise{
  Our approach is extensively evaluated on 10 real-world datasets. Experimental results show that $\text{C}^{2}\text{TC}$ achieves at least \textbf{2 orders of magnitude} speedup over state-of-the-art DC methods, while delivering superior downstream performance in both balanced and imbalanced settings.
  }
  }
\end{itemize}


\myblue{\textit{Note that, due to space limitations, all proofs, some related work, experimental settings, and additional evaluations can be found in the online appendix~\cite{appendix}.}}

\section{Related Work}

\myparagraph{Dataset condensation}  
Dataset condensation has recently become a popular data-centric solution for efficient deep learning by generating compact yet informative synthetic datasets preserving the training utility of the original data.
Existing DC studies mainly fall into two categories.  
\textit{Parameter matching} methods align parameter updates between real and synthetic data, either in a single step or across multiple steps.  
Single-step strategies, such as Gradient Matching~\cite{zhao2021dataset,lee2022dataset,jiang2023delving,kim2022dataset}, match gradients within one update, simplifying optimization but remaining constrained by bi-level training.  
Multi-step frameworks like Matching Training Trajectories~\cite{cazenavette2022dataset,guotowards,zhong2025towards} perform parameter alignment over multiple update iterations, achieving higher fidelity but requiring numerous expert networks, which makes the process highly time- and resource-intensive.
\textit{Distribution matching} methods eliminate redundant model training by aligning feature-level statistics between real and synthetic data through model embeddings, such as class-wise mean matching~\cite{zhao2023dataset} or style–content alignment~\cite{malakshan2025decomposed}.
However, they remain inherently model-dependent, leading to limited cross-architecture generalization, and still require gradient computation and backpropagation to update the synthetic data, incurring considerable computational overhead.
Despite recent progress, existing dataset condensation methods overlook the distinctive characteristics of tabular data, such as heterogeneous features and class imbalance, and struggle to scale to large datasets, which is particularly critical since tabular data, as the mainstream format in relational databases, are often large in scale. 
\red{Recently, Gao et al. [1] proposed training-free graph condensation through theoretical analysis. However, our work fundamentally differs in problem definition, theoretical framework, and label allocation strategies, with specific comparisons detailed in Appendix~\ref{cgc}~\cite{appendix}.}

A more detailed discussion of coreset selection methods is provided in the Appendix~\ref{cs_rw}~\cite{appendix}.

\section{Preliminaries}

\subsection{Tabular Data Condensation}

\myparagraph{Notations} Consider a general classification problem on tabular data, we denote the dataset as
$\mathcal{T}=\{(x_{i},y_{i})\}_{i=1}^{N}$, which can be viewed as a table with 
$N$ rows and $F$ columns. Each instance $x_{i} \in \mathbb{R}^{F}$ corresponds to a row of the table with 
$F$-dimensional features, and each label $y_{i} \in \mathbb{R}^{C}$ is a one-hot vector indicating the ground-truth class from the set $\mathcal{C} = \{0,1,...C-1\}$. 
Let $x_{i,j}$ be the $j$-th feature of the instance $x_{i}$. This feature can be either a numerical (continuous) attribute, denoted as $x_{i,j}^{num} \in \mathbb{R}$, or
a categorical (discrete) value $x_{i,j}^{cat}$. In practice, categorical features are typically stored as strings or integers in the raw data.
We denote $X=[x_{1},x_{2},...,x_{N}]^{T} \in \mathbb{R}^{N \times F}$ as the feature matrix and $Y=[y_{1},y_{2},...,y_{N}]^{T} \in \mathbb{R}^{N \times C}$ is the label matrix.
Let $X_{i}$ be the set of samples belonging to class $i$ in $\mathcal{T}$, and let $n_{i}=|X_{i}|$ denote its cardinality, i.e., the number of samples in class $i$. Accordingly, $\{n_{i}\}_{i=0}^{C-1}$ represent the per-class sample sizes.

\myparagraph{Problem definition} Given the input large-scale tabular dataset $\mathcal{T}=(X,Y)$, the goal of tabular data condensation is to generate a downsized table $\mathcal{S}=(X^{\prime},Y^{\prime})$ with $X^{\prime} \in \mathbb{R}^{N^{\prime} \times F}$ and $Y^{\prime} \in \mathbb{R}^{N^{\prime} \times C}$ ($N^{\prime} \ll N$), such that the deep tabular model trained 
on $\mathcal{S}$ can achieve comparable performance to that trained on the original 
$\mathcal{T}$.
Formally, tabular data condensation can be defined as a bi-level optimization problem:
\begin{equation}
\label{eq:problem_definition}
\begin{split}
       \min_{\mathcal{S}} \; &\mathcal{L}(f_{\theta_{\mathcal{S}}}(X),Y) \\
        \text{s.t.}  \quad  \theta_{\mathcal{S}} = & \argmin_{\theta} \; \mathcal{L}(f_{\theta}(X^{\prime}),Y^{\prime}),
\end{split}
\end{equation}
where $\mathcal{L}$ is the loss function. 
$f_{\theta}$ represents the deep tabular model parameterized with $\theta$, and $\theta_{\mathcal{S}}$ denotes the parameters of the model trained on the downsized synthetic table. 


\subsection{Dataset Condensation Methods}
\label{sec:3.2}
\myblue{To address the dataset condensation problem formulated in Eq.~(1), 
existing methods generally treat the synthetic dataset $\mathcal{S}$ as a set of learnable variables and optimize it to approximate the training behavior of the original dataset.
The condensation process typically proceeds in two stages.
First, $\mathcal{S}$ is initialized using strategies such as random noise (e.g., Gaussian) or by sampling a random subset from the original dataset $\mathcal{T}$.
In practice, the subset-based initialization is more commonly used, since it offers a more stable starting point for optimization~\cite{cui2022dc,yu2023dataset}.
Following initialization, $\mathcal{S}$ is iteratively updated via a relay model by minimizing a matching objective, which enforces consistency between the training dynamics of models trained on the original and condensed datasets (i.e., parameter matching), or alignment in the embedding space (i.e., distribution matching). These matching losses are backpropagated to update the learnable condensed data $\mathcal{S}$ through gradient-based optimization, thereby preserving the training utility of the original large-scale data.
Existing dataset condensation methods primarily distinguish themselves by how the matching objective is formulated. We next present the two most representative optimization strategies: parameter matching and distribution matching.
}

\myparagraph{Parameter matching (PM)} 
The core idea of parameter matching is that the parameters of deep models should be well-aligned when trained on either $\mathcal{T}$ or $\mathcal{S}$.
To achieve this, \textit{Gradient Matching} (GM) based methods~\cite{zhao2021dataset,lee2022dataset,jiang2023delving,kim2022dataset} have been proposed, which align the gradient directions of the loss function during optimization:
\begin{equation}
\label{eq:gm_loss}
\mathcal{L}_{GM} = \mathbb{E}_{\theta \sim \Phi} \left[ \sum_{i=0}^{C-1} \mathcal{D} \left( \nabla_{\theta} \mathcal{L}_i^{\mathcal{S}}, \nabla_{\theta} \mathcal{L}_i^{\mathcal{T}} \right) \right],
\end{equation}
where $\Phi$ denotes the distribution of relay model parameters $\theta$.
$\mathcal{D}(\cdot,\cdot)$ denotes a distance function, while $\mathcal{L}_i^{\mathcal{S}}$ and $\mathcal{L}_i^{\mathcal{T}}$ represent the class-wise classification losses on the synthetic dataset $\mathcal{S}$ and the real dataset $\mathcal{T}$, respectively. For simplicity, the update of the relay model in the inner loop is omitted. 
However, GM-based methods suffer from error accumulation when models are trained on $\mathcal{S}$ for multiple steps, leading to suboptimal performance. 
To address this issue, \textit{Matching Training Trajectories} (MTT)~\cite{cazenavette2022dataset} was introduced, which performs long-range parameter matching between training on synthetic and real data, and has attracted considerable research attention~\cite{guotowards,liu2024dataset,zhong2025towards}.
This strategy first trains deep models on the original dataset to obtain and store expert parameter trajectories. During condensation, the synthetic dataset is optimized via multi-step parameter matching, which aligns model parameters trained on synthetic data with pre-recorded expert trajectories. 



\begin{equation}
\begin{aligned}
\mathcal{L}_{\mathit{MTT}} = \mathbb{E}_{\theta^{(0)} \sim \Phi}& \left[ \mathcal{D} \left( \theta_{\mathcal{S}}^{(T_s)}, \theta_{\mathcal{T}}^{(T_t)} \right) \right], \\
\text{where} \quad \mathcal{D}(\theta_{\mathcal{S}}^{(T_s)}, \theta_{\mathcal{T}}^{(T_t)}) 
&= \frac{ \| \theta_{\mathcal{S}}^{(T_s)} - \theta_{\mathcal{T}}^{(T_t)} \|^2 }
       { \| \theta_{\mathcal{T}}^{(T_t)} - \theta^{(0)} \|^2 }.
\end{aligned}
\end{equation}
Here, $\theta^{(0)}$ is the shared initialization point. 
Specifically, the model is optimized for $T_s$ steps on $\mathcal{S}$ and $T_t$ steps on $\mathcal{T}$, and the matching loss drives the synthetic trajectory to align the expert trajectory at their respective endpoints.
However, avoiding overfitting to a single initialization of the relay model necessitates training hundreds of deep neural networks on $\mathcal{T}$ to extract trajectories, incurring heavy computational costs.

\myparagraph{Distribution matching (DM)} DM proposes to align the class-conditioned feature distributions between the original and condensed datasets. 
This strategy is first introduced in \cite{zhao2023dataset} by minimizing the discrepancy
between class prototypes (i.e., class-wise mean feature embeddings):
\begin{equation}
\label{eq:dm_loss}
\mathcal{L}_{\mathit{DM}} = \mathbb{E}_{\theta \sim \Phi} \left[ \left\| \mathbf{P}  \Phi_{\theta}(X) - \mathbf{P}' \Phi_{\theta}(X') \right\|^2 \right],
\end{equation}
where the matrices $\mathbf{P} \in \mathbb{R}^{C \times N}$ and $\mathbf{P}' \in \mathbb{R}^{C \times N'}$ are class-wise mean pooling operators defined as:
\[
\mathbf{P}_{i,j} = 
\begin{cases}
\frac{1}{n_i}, & \text{if } y_j = i \\
0, & \text{otherwise}
\end{cases}, \quad
\mathbf{P}'_{i,j} = 
\begin{cases}
\frac{1}{n'_i}, & \text{if } y'_j = i \\
0, & \text{otherwise}
\end{cases}.
\]
Here, $n_i$ and $n'_i$ denote the number of samples in class $i$ from the real dataset $\mathcal{T}$ and the synthetic dataset $\mathcal{S}$, respectively.
By eliminating the need to compute model gradients, DM enables a more efficient and flexible condensation process, which has made it prevalent in recent DC studies. 
For example, IDM~\cite{zhao2023improved} further extends this idea by leveraging validated embeddings and incorporating class-aware regularization, while Malakshan et al.~\cite{malakshan2025decomposed} decompose the DM-based image data condensation into style and content matching for better distillation.

\section{Training-free \revise{Framework} with \\ Class-adaptive Partition}
\label{sec:objective}

\subsection{Bridging Condensation Objectives}
\label{sec:4.1}

In this section, we theoretically reveal the underlying connections among the matching objectives of the aforementioned DC methods (Section~\ref{sec:3.2}).
In our analysis, we consider the ridge regression optimization objective for model training on the original dataset $\mathcal{T}$ and the synthetic dataset $\mathcal{S}$ as follows:
\begin{equation}
\begin{split}
  \mathcal{L}^{\mathcal{T}} &= \frac{1}{2} \left\| \Phi_{\theta}(X)\mathbf{W} - Y \right\|^2 + \lambda \|\mathbf{W}\|^2, \\
  \mathcal{L}^{\mathcal{S}} &= \frac{1}{2} \left\| \Phi_{\theta}(X^{\prime})\mathbf{W} - Y^{\prime} \right\|^2 + \lambda \|\mathbf{W}\|^2,
\end{split}
\end{equation}
where $\mathbf{W} \in \mathbb{R}^{d \times C}$ is the parameter of the classification layer and $\Phi_{\theta}: \mathbb{R}^{F} \rightarrow \mathbb{R}^{d}$ is the feature encoder parameterized by $\theta$, with $d$ denoting the embedding dimension. For simplicity, we assume that only $\mathbf{W}$ is updated during condensation, while $\Phi_{\theta}$ remains fixed. 
\myblue{We adopt ridge regression as the training objective for theoretical analysis due to its analytical tractability and interpretability. Its closed-form solution and explicit gradients allow us to transparently analyze optimization dynamics and establish a clear connection between distribution matching and gradient matching objectives. Moreover, ridge regression provides a meaningful analytical lens for tabular data, where features correspond to semantically well-defined attributes~\cite{khurana2021semantic} and second-order statistics (e.g., covariance) capture essential global feature correlations~\cite{hastie2009elements,bishop2006pattern}. As the solution of ridge regression is explicitly determined by these statistics, it enables principled analysis of how condensation objectives preserve fundamental feature relationships. This choice is made for theoretical clarity rather than to position ridge regression as an optimal predictor.}
The following theorem establishes the connection between gradient- and trajectory-matching objectives, the two key formulations of the parameter-matching paradigm.

\begin{theorem}
\label{thm:gm-mtt}
Let $\mathcal{T}$ and $\mathcal{S}$ denote the real and synthetic datasets, respectively. Suppose the model is initialized with shared parameters $\mathbf{W}_{0}$, and trained separately on both datasets for $T$ steps using gradient descent with a fixed
learning rate $\eta > 0$. Then, after $T$ steps, the parameter deviation is bounded by:
\begin{equation}
\|\mathbf{W}_{T}^{\mathcal{S}} - \mathbf{W}_{T}^{\mathcal{T}}\| \leq \eta \sum_{t=0}^{T-1} \|\nabla \mathcal{L}^{\mathcal{S}}(\mathbf{W}_{t}^{\mathcal{S}}) - \nabla \mathcal{L}^{\mathcal{T}}(\mathbf{W}_{t}^{\mathcal{T}})\|,
\end{equation}
where $\nabla \mathcal{L}^{\mathcal{S}}(\mathbf{W}_{t}^{\mathcal{S}})$ and $\nabla \mathcal{L}^{\mathcal{S}}(\mathbf{W}_{t}^{\mathcal{T}})$ denote $t$-th step gradient on $\mathcal{S}$ and $\mathcal{T}$, respectively.
\end{theorem}

Theorem~\ref{thm:gm-mtt} indicates that if gradient matching is enforced at each step (i.e., the gradients from real and synthetic datasets are aligned), then the resulting model parameter trajectories remain closely aligned. Therefore, the gradient matching objective can be viewed as an implicit surrogate for the trajectory matching objective.
\myblue{In addition, Theorem~\ref{thm:dm-gm} establishes the connection between distribution matching and gradient matching based on the propositions in~\cite{yu2023dataset,gao2025rethinking}.}

\begin{theorem}
\label{thm:dm-gm}
Second-order constrained distribution matching yields an upper bound on gradient matching. 
\end{theorem}

\begin{remark}
Based on Theorem~\ref{thm:gm-mtt}-\ref{thm:dm-gm}, we observe that various condensation objectives can be unified under the framework of class-level feature alignment between the original and condensed datasets, namely, distribution matching.
\end{remark}

However, this class-to-class paradigm shown in Eq.~(\ref{eq:dm_loss}) yields a coarse optimization target at the sample level. To address this, we extend the DM objective by introducing a class-to-sample matching scheme, enabling more fine-grained and interpretable optimization that better preserves sample-level information during condensation.


\subsection{A Unified Class-wise Partition Objective}
\label{sec:4.2}


For distribution matching objective formulated in Eq.~(\ref{eq:dm_loss}), to facilitate a class-to-sample matching
paradigm, we refine the aggregation matrix $\mathbf{P}$ from the following two aspects:

\begin{itemize}[leftmargin=*]
  \item The number of aggregated embeddings increases from $C$ to $N'$, ensuring that each condensed sample corresponds to the aggregation of a specific subset of original samples;
  \item Aggregations are performed within each class to preserve label semantics at a finer granularity.
\end{itemize}
Subsequently, we show that the above class-to-sample scheme can be formulated as the following class partition problem.
\begin{definition}[Class Partition Problem]
  The class partition divides $n_{i}$ original sample embeddings in class $i$ into $n^{\prime}_i$ disjoint subsets $\{ S_1^i, ..., S_{n^{\prime}_i}^i \}$, each represented by a centroid aggregated by its constituent embeddings. Let $\pi^i : \{1, ..., n_i\} \rightarrow \{1, ..., n^{\prime}_{i}\}$ denote the assignment function that maps each sample to its corresponding subset.
  The class-wise aggregation matrix $\mathbf{R}^i \in \mathbb{R}^{n^{\prime}_{i} \times n_{i}}$ is defined as:
\begin{equation}
\mathbf{R}^i_{j,k} =
\begin{cases}
\frac{1}{|S_j^i|}, & \text{if } \pi^i(k) = j,\\[4pt]
0, & \text{otherwise}
\end{cases},
\end{equation}
\end{definition}
where $\mathbf{R}^i_{j,k}$ denotes the aggregation weight for sample $k$ in class $i$, and $|S_j^i|$ is the size of subset $S_j^i$. 
Based on the above defined partition problem, the aggregation matrix $\mathbf{P} \in \mathbb{R}^{C \times N}$ is updated to $\hat{\mathbf{P}} \in \mathbb{R}^{N' \times N}$ by stacking all class-wise aggregation matrices along the diagonal, i.e., $\hat{\mathbf{P}} = \text{diag}(\mathbf{R}^1, ..., \mathbf{R}^C)$. Assuming that condensed samples are ordered in ascending order based on class labels, $\mathbf{P}'$ degrades to the identity matrix $\mathbf{I}$, and the DM objective in Eq.~(\ref{eq:dm_loss}) can be reformulated as:
\begin{equation}
\label{eq:update_dm_loss}
\operatorname*{arg\,min}_{X^{\prime},\hat{\mathbf{P}}}
\, \mathbb{E}_{\theta \sim \Phi} \left[ \left\| \hat{\mathbf{P}} \Phi_{\theta}(X) - \Phi_{\theta}(X') \right\|^2 \right],
\end{equation}
\myblue{Directly solving the above class partition problem with a nonlinear encoder is intractable in practice, as it requires enumerating a high-dimensional parameter space and computing pairwise distances across all samples, rendering the problem NP-hard~\cite{mirzasoleiman2020coresets}. To make the formulation both analytically and practically feasible, we linearize the feature encoder. This approximation is justified below.
(1) For tabular data, features typically correspond to semantically well-defined attributes~\cite{khurana2021semantic} (e.g., age, education, occupation in the Adult dataset~\cite{kohavi1996scaling}), and a linear encoder preserves interpretable statistical relationships among such attributes~\cite{hastie2009elements,bishop2006pattern}, whereas nonlinear encoders may distort feature semantics during optimization. 
(2) Moreover, with the linear encoder, our condensation procedure is not tied to any specific nonlinear deep architecture and does not rely on inductive biases introduced by deep encoders. This design choice distinguishes our method from PM and DM that depend heavily on nonlinear models, and it contributes to stronger downstream architecture transferability.}
\revise{This 
linear surrogate model 
offers a transparent and semantics-preserving mapping, formulated as $\Phi_{\Theta}(X)=X\Theta$.
Then, Eq.~(\ref{eq:update_dm_loss}) can be written as:
\begin{equation}
\label{eq:linear_dm_loss}
\operatorname*{arg\,min}_{X^{\prime},\hat{\mathbf{P}}}
\, \mathbb{E}_{\Theta \sim \Phi} \left[ \left\| \hat{\mathbf{P}} X \Theta - X'\Theta \right\|^2 \right].
\end{equation}
which is upper bounded by
\begin{equation}
\label{eq:linear_dm_loss_ub}
\mathbb{E}_{\Theta \sim \Phi} \left[ \left\| \hat{\mathbf{P}} X \Theta - X'\Theta \right\|^2 \right] \leq \mathbb{E}_{\Theta \sim \Phi} \left[ \left\| \hat{\mathbf{P}} X - X' \right\|^2 \left\| \Theta \right\|^{2} \right].
\end{equation}
Given that $\Theta$ is independent of $\hat{\mathbf{P}},X^{\prime}$, we can minimize the upper bound by optimizing
\begin{equation}
\arg\min_{X^{\prime}, \hat{\mathbf{P}}} \|\hat{\mathbf{P}}  X -X'\|^2.
\label{eq:input-objective}
\end{equation}
This objective indicates that the condensed table $X^{\prime}$ can be obtained by performing a class partition on the original tabular data $X$, eliminating the gradient-based distribution matching optimization. This model-agnostic objective can be efficiently solved by applying
any Expectation-Maximization (EM) based clustering algorithms (e.g., K-means) to each class, iteratively updating the cluster centroid $X^{\prime}$
and the aggregation matrix $\hat{\mathbf{P}}$ until convergence. 
}

\subsection{Class-adaptive Cluster Allocation Problem}
\label{sec:4.3}
As discussed in Section~\ref{intro}, pre-defined label distribution strategies for synthetic data may overemphasize either minority or majority classes, reducing the utility of the condensed data. To address this, we enhance our training-free partition-based condensation objective in Eq.~(\ref{eq:input-objective}) by incorporating adaptive per-class sample allocation, which jointly optimizes the number of condensed samples per class and their intra-class clustering quality.
This leads to the following \textit{Class-adaptive Cluster Allocation Problem (CCAP)}.
\begin{definition}[Class-adaptive Cluster Allocation Problem]
Given the original dataset $\mathcal{T}=(X,Y)$, the per-class sample sizes $\{n_{i}\}$, the class set $\mathcal{C}$, the number of total clusters $N^{\prime}$, a clustering loss function $\mathcal{L}(\cdot)$, and a reweighting exponent $\gamma$, the objective of class-adaptive cluster allocation problem is to  
determine per-class cluster count $n_{i}^{\prime}$ and the corresponding optimal partition subsets $\{S^{i}_{j}\}^{*}$ for each class $i$, such that the following 
objective is minimized: 
\begin{equation}
\small
\label{objective}
\begin{aligned}
    \min_{\{n_{i}^{\prime}\}} \quad & 
    \sum_{i=0}^{C-1} \frac{1}{n_{i}^{\gamma}} \mathcal{L}(\{S_{j}^{i}\}^{*}) \\[3pt]
    \text{s.t.} \quad 
    & \{S_{j}^{i}\}^{*} = \argmin_{\{S_{j}^{i}\}} \mathcal{L}(\{S_{j}^{i}\}), \; \; \forall i \in \mathcal{C}, \\[3pt]
    & 1 \leq n_i^{\prime} \leq \min\{ N^{\prime}-(C-1),\, n_{i} \}, \; \; \forall i \in \mathcal{C}, \\[3pt]
    & \sum_{i=0}^{C-1} n_i^{\prime} = N^{\prime}.
\end{aligned}
\end{equation}

where ${n_{i}^{\prime}}$ denotes the number of clusters assigned to each class,
$\mathcal{L}(\{S_{j}^{i}\})$ denotes the clustering loss given the partition subsets $\{S_{j}^{i}\}$ with $n_{i}^{\prime}$ clusters. 
\end{definition}
The rationale behind Eq.~(\ref{objective}) is as follows. It adopts a bi-level structure, where the upper-level determines the allocation of class-wise cluster number $\{n_{i}^{\prime}\}$, and the lower level solves clustering subproblems within each class to minimize the corresponding partition loss. As the inner-level problem can be efficiently solved using standard clustering algorithms (e.g., K-means) given a fixed number of clusters, the overall problem essentially reduces to combinatorial optimization over discrete integer variables $\{n_{i}^{\prime}\}$, aiming to achieve minimal total clustering loss across all classes.
However, as class imbalance is prevalent in tabular data, majority classes tend to dominate the sample size and thus contribute disproportionately to the clustering loss.
Consequently, the optimization is biased toward improving the clustering quality of majority classes, which leads to allocating more clusters to them and further amplifies the imbalance in the condensed data. To mitigate this issue,
we introduce a class-wise scale factor. Specifically, the optimal clustering loss is reweighted by $\frac{1}{n_{i}^{\gamma}}$, where $\gamma \in [0,1]$ is an adaptive weighting exponent. When $\gamma=1$, all classes are assigned equal weights, ensuring maximum balance, whereas 
$\gamma = 0$ yields proportional weighting, which places greater emphasis on majority classes. To ensure feasibility and class coverage during condensation, we impose two integer constraints: (1) the total number of class-specific clusters must equal the cluster budget $N^{\prime}$; and (2) each class is assigned at least one cluster, but no more than its number of samples in the original dataset, thereby preventing class dropout and over-clustering.
In this work, we adopt the K-means algorithm for clustering, which minimizes the within-cluster sum of squares (WCSS)~\cite{lloyd1982least} as the clustering loss. 

\begin{theorem}
\label{thm:pgsa-nphard}
\revise{The class-adaptive cluster allocation problem (CCAP) is NP-hard.}
\end{theorem}

To demonstrate the difficulty of the formulated CCAP, we analyze the hardness of the problem in Theorem~\ref{thm:pgsa-nphard}.
We prove the NP-hardness of CCAP by a reduction from the Multiple-Choice Knapsack Problem (MCKP)~\cite{Kellerer2004Knapsack}.

\section{The Proposed Approach}
\label{sec:approach}
\revise{
In this section,
we propose the $\mathrm{C^{2}TC}$ framework, which consists of two main components, \red{as illustrated in Figure~\ref{fig:pipeline}}.
We first introduce a hybrid categorical feature encoding (HCFE) strategy that transforms heterogeneous features into a unified numerical space while preserving their semantic information (Section~\ref{cate_enc}). 
We then propose a Heuristic First-Improvement Local Search (HFILS) algorithm developed to efficiently solve the CCAP problem (Section~\ref{subsec:first-improvement}). Guided by the elbow heuristic, HFILS iteratively alternates between soft allocation and class-wise clustering, where the former adaptively generates feasible cluster reallocations and the latter evaluates their clustering quality.
By updating the allocation whenever the objective improves, HFILS achieves efficient optimization while ensuring high-quality condensation.
}

\label{framework}
\begin{figure} [t]
    \centering
    \includegraphics[width=\linewidth]{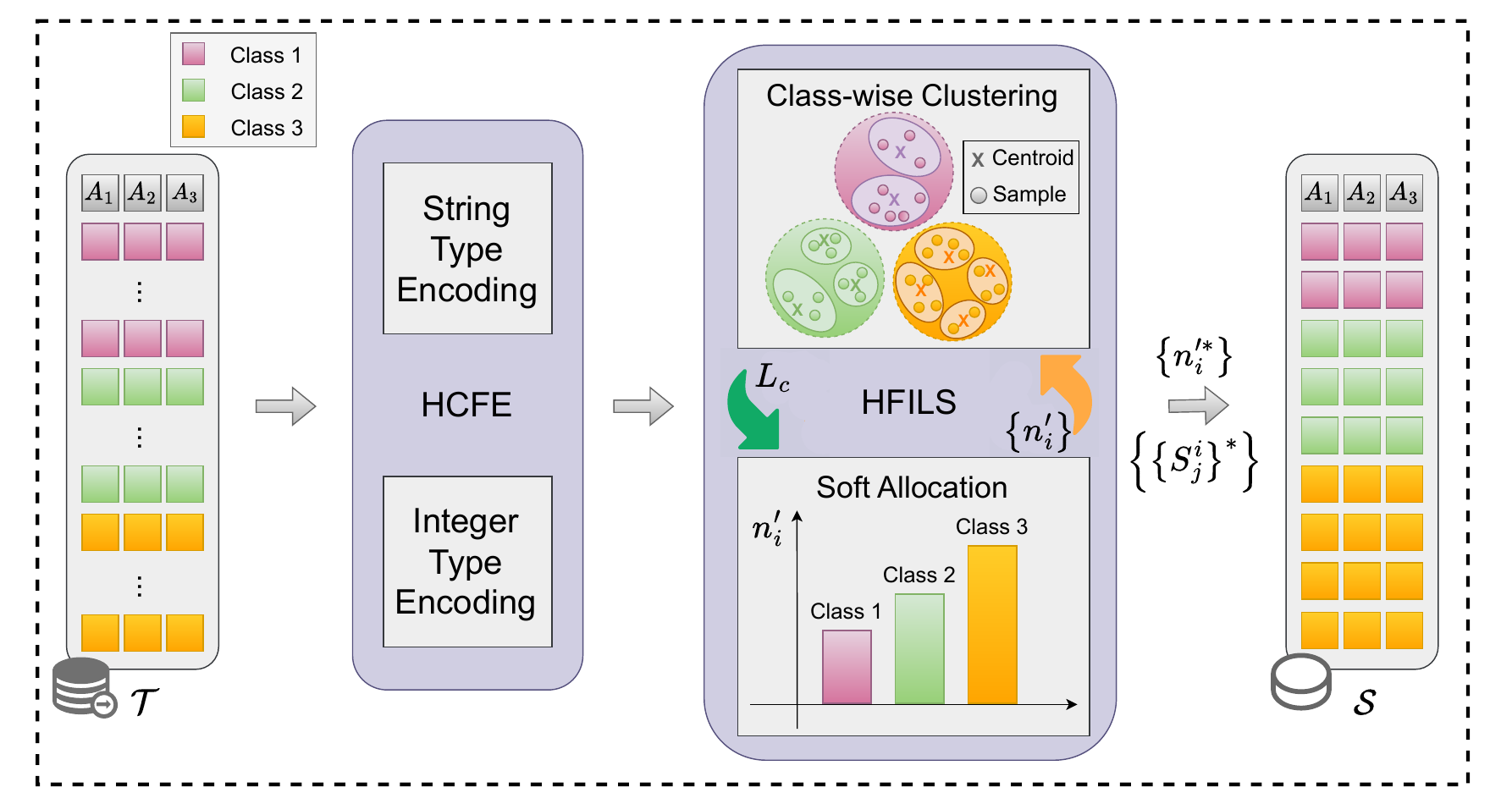}
     \caption{Overall framework of $\mathrm{C^{2}TC}$.}
    \label{fig:pipeline}
\end{figure}




\subsection{Hybrid Categorical Feature Encoding}
\label{cate_enc}
\revise{
Categorical features in tabular data are typically represented as strings or integers, yet they lack inherent ordering and a well-defined distance metric, making it nontrivial to incorporate their semantics into distance-based clustering. The difficulty is further compounded by the heterogeneity of categorical features, which requires representations that remain semantically faithful while being numerically comparable across dimensions. Conventional categorical encoding approaches, such as one-hot encoding~\cite{borisov2022deep}, expand these features into high-dimensional sparse vectors, which exacerbate the curse of dimensionality and lead to unstable clustering, increased computational overhead, and reduced interpretability~\cite{high-d}. To address these intrinsic limitations, we develop HCFE, a hybrid encoding strategy that unifies heterogeneous categorical representations into a low-dimensional numerical space, 
effectively preserving semantic consistency while ensuring compatibility with clustering-based distance computations.
}

\myparagraph{String-type categorical feature encoding}
\myblue{String-type categorical features (e.g., ``NSW'', ``Doctor'', ``Male'') often exhibit rich lexical structures that traditional one-hot encoding fails to capture. One-hot encoding treats all categories as orthogonal and equidistant, thereby discarding the intrinsic relationships between string values (e.g., "Engineer" vs. "Senior Engineer").
To address this, we employ similarity encoding~\cite{cerda2018similarity} based on n-gram overlaps~\cite{li2007vgram}. This method explicitly maps discrete tokens into a continuous space where distances reflect lexical and semantic relatedness, making it robust to variations in string representations.}
Formally, let \( s \in \mathbb{S}^n \) denote a categorical column with \( k \) unique string values 
\( \{ d_1, d_2, \dots, d_k \} \). 
Given a symmetric similarity function 
\( \text{sim}: \mathbb{S} \times \mathbb{S} \to [0, 1] \), 
where \( \text{sim}(d_i, d_j) \) measures the n-gram overlap~\cite{li2007vgram} between two strings, 
we map each value $d_i$ to a $k$-dimensional 
vector:
\begin{equation}
\mathbf{v}_i = \left[\text{sim}(d_i, d_1), \, \text{sim}(d_i, d_2), \, \dots, \, \text{sim}(d_i, d_k) \right].
\end{equation}
\myblue{While similarity encoding captures semantic information, determining the full similarity matrix for high-cardinality features results in a high-dimensional and redundant representation. To address this, we concatenate the similarity vectors from all string-type categorical columns into a single embedding vector 
$\mathbf{x}_{\text{str}} \in \mathbb{R}^{D_{\text{str}}}$,
where $D_{\text{str}} = \sum_{j=1}^{|S_{\text{str}}|} K_j$.
Here, $S_{\text{str}}$ denotes the set of string-type categorical features, and $K_j$ is the cardinality of the $j$-th feature in $S_{\text{str}}$.
This high-dimensional input is then mapped by the encoder
$E_{\mathcal{W}_E}(\cdot)$ to a compact latent representation
$\mathbf{z}_{\text{str}} \in \mathbb{R}^{|F_{\text{str}}|}$.
Here, we set the latent dimension equal to the number of original string-type
categorical features to preserve the feature count while removing redundancy.
Subsequently, the decoder $D_{\mathcal{W}_D}(\cdot)$ reconstructs the original
embedding from $\mathbf{z}_{\text{str}}$.
The model is optimized by minimizing the reconstruction loss:
\begin{equation}
\mathcal{L}_{\text{AE}} =
\frac{1}{N} \sum_{i=1}^{N}
\left\|
\mathbf{x}_{\text{str}, i}
- D_{\mathcal{W}_D}\!\left(
E_{\mathcal{W}_E}(\mathbf{x}_{\text{str}, i})
\right)
\right\|_2^2,
\end{equation}
where $N$ denotes the number of training samples.
This objective ensures that the latent representation preserves the semantic
structure captured by the similarity encoding.
The autoencoder is trained by backpropagating the reconstruction loss
$\mathcal{L}_{\text{AE}}$ and updating the parameters of both
$E_{\mathcal{W}_E}$ and $D_{\mathcal{W}_D}$ via gradient descent.
After training, the final encoder $E_{\mathcal{W}_E}$ is used for feature mapping.}
The resulting embeddings are then rescaled to the $[0,1]$ via Min–Max normalization~\cite{al2006data}, yielding semantically consistent and numerically compatible representations for downstream clustering.

\myparagraph{Integer-type categorical feature encoding}
\myblue{In many tabular benchmarks, categorical attributes are provided as integer codes, where the original semantic meanings are often inaccessible. 
For binary attributes, we retain the $\{0,1\}$ label encoding~\cite{borisov2022deep}, which naturally lies within the $[0,1]$ range and can be directly utilized in distance computations for clustering. 
However, for multi-valued attributes, 
treating them as numerical values introduces spurious ordinal relationships (e.g., implying category "5" is greater than "1"), which distorts distance metrics in clustering. Additionally, one-hot encoding is often impractical for high-cardinality features due to the curse of dimensionality. To address this, we adopt a smoothed variant of target encoding~\cite{micci2001preprocessing,borisov2022deep}, which substitutes the implicit integer codes with the mean of the target variable. This effectively injects supervised statistical information into the feature representation, converting discrete, meaningless integer codes into a continuous space that reflects the category's predictive signal.}
Formally, for a sample \( x_i \) with label \( y_i \), let \( x_i^{(j)} \) denote the \( j \)-th categorical feature with \( K_j \) distinct values. 
The smoothed target encoding is formulated as:
\begin{equation}
\small
\begin{split}
\tilde{x}_{i}^{(j)} & = \frac{m_{i} \cdot \mu(\{y_{i}\}) + \lambda \cdot \mu(\{y\})}{m_{i} + \lambda} \\
\mu(\{y_{i}\}) = & \frac{1}{m_{i}} \sum_{k:x_{k}^{(j)} = x_{i}^{(j)}} y_{k}, \quad \mu(\{y\}) = \frac{1}{N}\sum_{k=1}^{N} y_{k},
\end{split}
\end{equation}
where \( m_i \) is the number of samples sharing the same value as \( x_i^{(j)} \), while $\mu(\{y_{i}\})$ and $\mu(\{y\})$ represent the category-specific and global label means, respectively.
The parameter 
$\lambda$
regulates the balance between local statistical fidelity and global stability during smoothing. 
To enhance generalization and mitigate overfitting, Gaussian noise is injected:
\begin{equation}
\tilde{x}_{i}^{(j)} \leftarrow \tilde{x}_{i}^{(j)} + \epsilon, \quad \epsilon \sim \mathcal{N}(0,\sigma^2).
\end{equation}
Finally, all encoded features are normalized to the $[0,1]$ range to ensure numerical comparability across heterogeneous attributes.
Overall, this strategy effectively captures statistical dependencies between categorical features and the target variable, yielding stable and informative numerical representations for distance-based clustering.

After the hybrid preprocessing, all heterogeneous features are represented in a unified numerical space, laying a foundation for effective and semantics-consistent data condensation.
\myblue{
The proposed condensation framework is designed to be encoder-agnostic. 
While HCFE is adopted as the default strategy, the framework can accommodate alternative numerical representations as demonstrated in 
\red{Appendix~\ref{ab_app}~\cite{appendix}}.
}

\subsection{Heuristic First-improvement Local Search}
\label{subsec:first-improvement}

\noindent\textbf{Motivation.} 
\revise{
Theorem~\ref{thm:pgsa-nphard} shows that CCAP is a generalized multiple-choice knapsack problem (MCKP), which is NP-hard.
Although MCKP can typically be solved by exact algorithms such as dynamic programming and branch-and-bound~\cite{Kellerer2004Knapsack}, CCAP is more challenging because the deterministic item values in MCKP correspond to clustering losses in CCAP that must be repeatedly computed through K-means across classes and cluster numbers, making the process computationally prohibitive.
Moreover, class allocations are interdependent, leading to an exponentially large combinatorial search space. Consequently, finding the globally optimal allocation is extremely challenging in both theory and practice.
To tackle this challenge, we draw inspiration from the \textit{elbow method} widely used in K-means, where clustering loss (i.e., WCSS) typically decreases as the number of clusters increases, but with diminishing marginal gains. This observation suggests that prioritizing additional clusters for minority classes
, which typically yield higher marginal gains in clustering loss reduction, can effectively enhance overall clustering quality.
Based on this intuition, we develop HFILS, which iteratively refines cluster allocations guided by the elbow-shaped loss trend and efficiently converges to a high-quality feasible solution.
}

\begin{algorithm}[t]
\SetVline
\footnotesize
\caption{{\textsc{HFILS}}}
\label{alg:ifins}
\Input{Original dataset $\mathcal{T}=(X,Y)$, number of classes $C$, condensation ratio $r$, max iterations $T$, adaptive exponent $\gamma$, step decay factor $l\!\in\!(0,1)$, tolerance $\epsilon$, patience $p$}
\Output{Final cluster allocation $\{n_i^{\prime *}\}$ and partitions 
$\{\{S^{i}_{j}\}^{*}\}$}
\BlankLine
\State{$n_i \gets |\,\{x \in X \mid y(x)=i\}\,|$ \textbf{for each} $i \in \{0,1,\dots,C-1\}$}

\For{$i \in \{0, 1, \ldots, C-1\}$}{
    \State{$n_i^{\prime} \gets \lfloor n_i \cdot r \rfloor$}
}
\State{$N^{\prime} \gets \sum_{i=0}^{C-1} n_i^{\prime}$}

\Cmt{Initialize loss cache matrix}
\State{$\mathbf{M} \in \mathbb{R}^{C \times N^{\prime}} \gets \infty$}

\State{$\mathcal{L}_c \gets \text{ClassWiseClustering}(\mathcal{T}, \{n_i^{\prime}\}, C, N^{\prime}, \gamma, \mathbf{M})$}

\State{$\mathcal{L}_c^{*} \gets \mathcal{L}_c$, $\{n_i^{\prime *}\} \gets \{n_i^{\prime}\}$}
\Cmt{Initialize allocation step}
\State{$s_{\max} \gets \max\!\bigl(\lfloor \mathrm{std}(\{n_i^{\prime}\}) \rfloor, 1\bigr)$} 
\Cmt{Early-stopping clock}
\State{$u \gets 0$} 

\For{$t \in \{0, 1, \ldots, T-1\}$}{
        \State{$\{n_i^{\prime}\} \gets \text{SoftAllocate}(\{n_i\}, \{n_i^{\prime}\}, C, N^{\prime}, s_{\max})$}
        \State{$\mathcal{L}_c \gets \text{ClassWiseClustering}(\mathcal{T}, \{n_i^{\prime}\}, C, N^{\prime}, \gamma, \mathbf{M})$}

        \If{$\mathcal{L}_c < \mathcal{L}_c^{*}$}{
            \State{$\mathcal{L}_c^{*} \gets \mathcal{L}_c$; $\{n_i^{\prime *}\} \gets \{n_i^{\prime}\}$}
            \Cmt{Step decay}
            \State{$s_{\max} \gets \max(\lfloor s_{\max} \cdot l \rfloor, 1)$} 
        }

        \If{$\mathcal{L}_c < \mathcal{L}_c^{*} - \epsilon$}{
            \State{$u \gets 0$}
        }
        \Else{
            \State{$u \gets u + 1$}
            \If{$u \ge p$}{
                \State{\textbf{break}}
            }
          
        }
    }
\State{$\{\{S^{i}_{j}\}^{*}\} \leftarrow$ RunKMeans($X_i$, $n_i'^*$) for each class $i$}

\Return{$\{n_{i}^{\prime *}\}, \{\{S^{i}_{j}\}^{*}\}$}
\end{algorithm}

\myparagraph{HFILS algorithm} 
HFILS serves as the core solver of our framework, iteratively searching and updating the candidate solution.
At each iteration, HFILS employs a soft allocation strategy (Algorithm~\ref{alg:softallocate}) guided by the elbow method to heuristically adjust per-class cluster numbers and achieve adaptive balance across classes through soft constraints. It then performs class-wise clustering (CW-Clus, Algorithm~\ref{alg:ccap}) to evaluate the clustering loss and update the allocation accordingly.

\revise{
The overall procedure is summarized in Algorithm~\ref{alg:ifins}.
We first adopt a ratio-preserving strategy (lines~2-3) to obtain an initial feasible allocation by proportionally distributing clusters across classes according to their sample sizes.
To accelerate later iterations, we introduce a \textit{loss cache matrix} $\mathbf{M}$ that stores clustering losses generated by CW-Clus, thus eliminating redundant computations (line~5).
The initial total loss is then evaluated through CW-Clus based on the current allocation (line~6).
To enable adaptive adjustment, we introduce a \textit{soft step size}, where the maximum step $s_{\max}$ is determined by the standard deviation (std) of the class distribution to reflect class imbalance (line~8). A real step size is subsequently sampled within this bound (line~5 in Algorithm~\ref{alg:softallocate}) to allow flexible and data-dependent reallocation.
The algorithm then proceeds iteratively, alternating between soft allocation and clustering evaluation (lines~11–12). Upon each iteration, the allocation is updated only when a first local improvement is detected (lines~13–14). 
\myblue{To facilitate efficient optimization, we adopt a \textit{step-decay mechanism} that gradually reduces the maximum step size \(s_{\max}\) via a scaling factor \(l \in (0,1)\) (line~15). This enables a coarse-to-fine exploration strategy: a larger \(s_{\max}\) early on allows broader exploration of class combinations, while progressively reducing \(s_{\max}\) restricts updates to fine-grained adjustments. This strategy stabilizes the condensation process and accelerates convergence.}
Finally, an \textit{early-stopping strategy} (lines~16-21) monitors convergence by tracking the relative improvement in the overall loss. The search terminates once this improvement falls below the threshold $\epsilon$ for $p$ consecutive iterations.
When either the maximum iteration $T$ is reached or the early-stopping criterion is satisfied, the best cluster allocation $
\{n_{i}^{\prime *}\}$ is obtained, followed by K-means clustering on each class to produce the corresponding intra-class partitions $\{\{S^{i}_{j}\}^{*}\}$ (line~22). The final allocation and partitions are returned (line~23).
}

\myparagraphexp{Class-wise clustering algorithm}
Algorithm~\ref{alg:ccap} illustrates the CW-Clus procedure.
For each class, the clustering loss is either computed via K-means (lines 3-8) or retrieved from the cached loss matrix $\mathbf{M}$ when the corresponding entry already exists (i.e., the condition in line 4 is not met).
The total clustering loss $\mathcal{L}_c$ is then accumulated over all classes (line 9) and returned as the overall objective value (line 10).

\begin{algorithm}[t]
{
\SetVline
\footnotesize
\caption{ClassWiseClustering}
\label{alg:ccap}
\Input{Original dataset $\mathcal{T}=(X,Y)$, number of classes $C$,  current cluster allocation $\{n_i^{\prime}\}$, total cluster count $N^{\prime}$, adaptive exponent $\gamma$, loss matrix $\mathbf{M}$}

\Output{Total clustering loss $\mathcal{L}_c$}

\vspace{3pt}
\State{$X_i \gets \{x \in X \mid y(x)=i\}$ \textbf{for each} $i \in \{0,1,\dots,C-1\}$}

\State{$\mathcal{L}_c \gets 0$}

\BlankLine
\For{$i \in \{0, 1, \ldots, C-1\}$}{
    \If{$\mathbf{M}[i, n_i^{\prime}] = \infty$}{
        \Cmt{Compute class-wise clustering loss}
        \State{$\textnormal{SSD},\{S^{i}_{j}\}^{*} \gets \text{RunKMeans}(X_i, n_i^{\prime})$}
        \State{$n_i \gets |X_i|$}
        \Cmt{Compute class-wise scale factor}
        \State{$w \gets \dfrac{1}{n_i^{\gamma}}$}
    
        \State{$\mathbf{M}[i, n_i^{\prime}] \gets \textnormal{SSD} \times w$}
    }
    \State{$\mathcal{L}_c \gets \mathcal{L}_c + \mathbf{M}[i, n_i^{\prime}]$}
}
\Return{$\mathcal{L}_c$}
}
\end{algorithm}

\myparagraphexp{Soft allocate algorithm}
To facilitate dynamic local search and better approximate the optimum of in Eq.~(\ref{objective}), we design a neighborhood scanning strategy based on composite reallocation moves.
In each move, the algorithm decreases the cluster count of one source class while simultaneously increasing those of multiple target classes, thereby realizing adaptive redistribution across classes. Instead of a rigid one-to-one adjustment where a single class gains clusters strictly at the expense of another, our approach employs a \textit{softer reallocation mechanism} that enables smoother, more flexible, and better-balanced redistribution across multiple classes.

\revise{
The soft allocation procedure is illustrated in Algorithm~\ref{alg:softallocate}. It first selects a source class $c_{\mathrm{src}}$ with more than one cluster (line~1) and constructs the target class set $\mathcal{C}_{\mathrm{tgt}}$, which contains classes that can still receive additional clusters within their upper limits (lines~2-4). A real step size $s$ is then randomly sampled from $[1, s_{\max}]$ and bounded by the transferable capacity of $c_{\mathrm{src}}$ (lines~5-6), introducing controlled stochasticity that promotes diverse yet feasible adjustments.
After subtracting $s$ clusters from the source class (line~7), the reallocation proceeds according to the remaining capacities of target classes. For each class in $\mathcal{C}_{\mathrm{tgt}}$, the algorithm estimates its residual capacity $r_c$ (lines~8-9) and determines a provisional allocation $a_c$ proportional to its relative ratio $\frac{r_c}{r_{\text{tot}}}$ (lines~11--12), guided by the elbow heuristic that prioritizes classes expected to yield higher clustering gains. To handle rounding effects, the remaining clusters $\delta$ are sequentially assigned to classes with larger residual capacities until all clusters are fully allocated (lines~13-19). Finally, the per-class cluster counts $n_c'$ are updated (lines~20--21) to reflect the completed redistribution, resulting in an updated allocation $\{n_i^{\prime}\}$ (line~22).
}

\begin{algorithm}[t]
\SetVline
\footnotesize
\caption{\text{SoftAllocate}}
\label{alg:softallocate}
\Input{Per-class sample sizes $\{n_i\}$, current cluster allocation $\{n_i^{\prime}\}$, number of classes $C$, total cluster count $N^{\prime}$, current max step size $s_{\max}$}
\Output{Updated per-class cluster sizes $\{n_i^{\prime}\}$}
\BlankLine

$c_{\mathrm{src}} \sim \text{Uniform}\!\left(\{\, i \mid n_i^{\prime} > 1 \,\}\right)$\;
\BlankLine

\For{$i \in \{0, 1, \ldots, C-1\}$}{
  \State{$n_{i,\max}^{\prime} \gets \min\{\,N^{\prime}-(C-1),\, n_i\,\}$}
}

$\mathcal{C}_{\mathrm{tgt}} = \{\, i \mid n_i^{\prime} < n_{i,\max}^{\prime} \ \land\ i \neq c_{\mathrm{src}} \,\}$\;
\BlankLine
\Cmt{Randomize a real step size}
$s \sim \textnormal{Uniform}(\{1,2,\dots,s_{\max} \})$\;
\BlankLine
$s \gets \min(s,\, n_{c_{\mathrm{src}}}^{\prime} - 1)$\;
\BlankLine
$n_{c_{\mathrm{src}}}^{\prime} \gets n_{c_{\mathrm{src}}}^{\prime} - s$\;
\Cmt{Allocate by remaining capacity ratio
}
\ForEach{$c \in \mathcal{C}_{\mathrm{tgt}}$}{ \State{$r_c \gets n_{c,\max}^{\prime} - n_c^{\prime}$}}
$r_{\text{tot}} \gets \sum_{c \in \mathcal{C}_{\mathrm{tgt}}} r_c$\;

\BlankLine
\ForEach{$c \in \mathcal{C}_{\mathrm{tgt}}$}{
  \State{$a_c \gets \left\lfloor s \cdot \dfrac{r_c}{\,r_{\text{tot}}\,} \right\rfloor$}
}

\BlankLine
\Cmt{Distribute unassigned clusters}
$\delta \gets s - \sum_{c \in \mathcal{C}_{\mathrm{tgt}}} a_c$\;
\textnormal{sort} $\mathcal{C}_{\mathrm{tgt}}$ by descending $r_c$\;
\ForEach{$c \in \mathcal{C}_{\mathrm{tgt}}$}{
  \If{$\delta = 0$}{\State{\textbf{break}}}
  \If{$n_c^{\prime} + a_c + 1 \le n_{c,\max}^{\prime}$}{
    $a_c \gets a_c + 1$;\quad \State{$\delta \gets \delta - 1$}
  }
}

\ForEach{$c \in \mathcal{C}_{\mathrm{tgt}}$}{ \State{$n_c^{\prime} \gets n_c^{\prime} + a_c$}}
\Return $\{n_i^{\prime}\}$
\end{algorithm}

\begin{theorem}[Time Complexity Analysis]
\label{thm:complexity}
Based on Algorithm~\ref{alg:ifins}, the time complexity of HFILS is $\mathcal{O}(T \cdot \frac{NN^{\prime}IF}{C})$, where $I$ is the number of iteration of K-means algorithm.
\end{theorem}

According to Theorem~\ref{thm:complexity}, the overall time complexity of the proposed HFILS algorithm scales linearly with both the size of the original dataset $N$, and the feature dimension $F$.

\section{Experiments} 
\label{sec:Experiments}

\subsection {Experimental Setup}
\label{exp_setups}

\noindent\textbf{Datasets.}
To evaluate the effectiveness and efficiency of our approach, we conduct experiments on 10 publicly available tabular datasets from OpenML benchmarks~\cite{vanschoren2014openml}. These datasets cover diverse real-world domains, including but not limited to energy forecasting, income prediction, and flight delay analysis, and encompass heterogeneous data types, such as numerical and categorical features. 
\myblue{More dataset description can be found in Appendix~\ref{dataset_app}~\cite{appendix}.}
For all datasets, we randomly split them for 80\%/10\%/10\% as train/validation/test sets.
Table~\ref{tab:datasets} provides the detailed statistics of all datasets.

\noindent\textbf{Baselines.}
We compare our method with representative baselines, including coreset selection approaches (Random, Herding~\cite{welling2009herding}, and K-Center~\cite{farahani2009facility}) and dataset condensation methods (DM~\cite{zhao2023dataset}, GM~\cite{zhao2021dataset}, and MTT~\cite{cazenavette2022dataset}).
\myblue{Detailed descriptions for the baselines can be found in Appendix~\ref{comp_method_app}~\cite{appendix}}.

\noindent\textbf{Evaluation metrics.}  
To evaluate the classification performance, we adopt two widely used metrics: \emph{Accuracy} and \emph{Macro-F1}.
\myblue{Detailed descriptions for the evaluation metrics can be found in Appendix~\ref{eval_app}~\cite{appendix}}.

\myblue{More implementation details of our experiments can be found in Appendix~\ref{implement_app}~\cite{appendix}}.
\begin{table}[t]
\caption{The statistics of datasets. Num and Cat stands for numerical and categorical features, respectively. MinMaxCF denotes the ratio of min--max class frequency~\cite{mcelfresh2023when}.}
\label{tab:datasets}
\centering
\small
\resizebox{\linewidth}{!}{\begin{tabular}{lccccc}
\toprule
\textbf{Dataset} & \textbf{Size} & \textbf{\#Num}  & \textbf{\#Cat} & \textbf{MinMaxCF} & \textbf{\#Classes} \\
\midrule
Electricity (EL) & 38,474  & 6   & 1   & 1.00 & 2 \\
Adult (AD) & 48,842  & 6   & 8   & 0.31 & 2 \\
Jannis (JA) & 83,733  & 54  & 0   & 0.04 & 4 \\
Diabetes130US (DA) & 101,766 & 13  & 36  & 0.21 & 3 \\
Road Safety (RS) & 111,762 & 29  & 3   & 1.00 & 2 \\
Epsilon (EP) & 500,000 & 2,000 & 0  & 1.00 & 2 \\
Airlines (AI) & 539,383 & 3   & 4   & 0.80 & 2 \\
Covertype (CO) & 581,012 & 10  & 44  & 0.01 & 7 \\
Higgs (HI) & 940,160 & 24 & 0  & 1.00 & 2 \\
Microsoft (MI) & 1,200,192 & 136 & 0 & 0.01 & 5 \\
\bottomrule
\end{tabular}}
\end{table}

\begin{table*}[t]
\caption{Overall Performance Comparison. \emph{Whole Dataset} denotes training on the full dataset and serves as an approximate upper bound. Bold indicates the
best performance and \underline{underline} means the runner-up.}
\vspace{-5pt}
\label{tab:class}
\centering
\small
\resizebox{1.0\textwidth}{!}
{%
\begin{tabular}{cccccccccccccccccc}
\toprule
\multirow{2}{*}{\textbf{Dataset}} & \multirow{2}{*}{r} & \multicolumn{2}{c}{\textbf{Random}} & \multicolumn{2}{c}{\textbf{Herding}}                      & \multicolumn{2}{c}{\textbf{K-Center}} & \multicolumn{2}{c}{\textbf{DM}} & \multicolumn{2}{c}{\textbf{GM}} & \multicolumn{2}{c}{\textbf{MTT}} & \multicolumn{2}{c}{\textbf{$\text{C}^{2}\text{TC}$}} & \multicolumn{2}{c}{\textbf{Whole Dataset}}                        \\ 
\cmidrule(lr){3-4}
\cmidrule(lr){5-6}
\cmidrule(lr){7-8}
\cmidrule(lr){9-10}
\cmidrule(lr){11-12}
\cmidrule(lr){13-14}
\cmidrule(lr){15-16}
\cmidrule(lr){17-18}
    &                    
    & Acc              & Macro-F1         
    & Acc                         & Macro-F1                    
    & Acc                & Macro-F1         
    & Acc            & Macro-F1       
    & Acc          & Macro-F1         
    & Acc             & Macro-F1       
    & Acc           & Macro-F1          
    & Acc                         & Macro-F1                    \\ 
\midrule
\multirow{3}{*}{EL}

  & 1\%    & 57.0\scriptsize{$\pm$2.5}
           & 56.0\scriptsize{$\pm$3.4}
           & 54.1\scriptsize{$\pm$0.1}
           & 54.1\scriptsize{$\pm$0.1}
           & 53.8\scriptsize{$\pm$1.5}
           & 48.5\scriptsize{$\pm$5.6}
           & 57.4\scriptsize{$\pm$2.0}
           & 56.4\scriptsize{$\pm$3.0}
           & 50.6\scriptsize{$\pm$1.3}
           & 37.0\scriptsize{$\pm$7.3}
           & \underline{58.0\scriptsize{$\pm$2.2}}
           & \underline{57.1\scriptsize{$\pm$2.9}}
           & \textbf{66.1\scriptsize{$\pm$0.5}}
           & \textbf{65.9\scriptsize{$\pm$0.5}}
           & \multirow{3}{*}{76.5\scriptsize{$\pm$0.9}}
           & \multirow{3}{*}{76.4\scriptsize{$\pm$0.9}} \\
  & 0.1\%  & 53.5\scriptsize{$\pm$2.3}
           & 53.0\scriptsize{$\pm$2.5}
           & \underline{54.2\scriptsize{$\pm$0.1}}
           & \underline{54.2\scriptsize{$\pm$0.1}}
           & 50.6\scriptsize{$\pm$3.7}
           & 48.4\scriptsize{$\pm$4.3}
           & 51.6\scriptsize{$\pm$2.4}
           & 51.0\scriptsize{$\pm$2.4}
           & 50.3\scriptsize{$\pm$2.7}
           & 44.3\scriptsize{$\pm$7.2}
           & 52.6\scriptsize{$\pm$3.4}
           & 51.9\scriptsize{$\pm$3.4}
           & \textbf{64.5\scriptsize{$\pm$0.5}}
           & \textbf{63.5\scriptsize{$\pm$0.6}}
           & 
           &  \\
  & 0.01\% & 51.2\scriptsize{$\pm$1.7}
           & 48.7\scriptsize{$\pm$4.5}
           & \underline{53.7\scriptsize{$\pm$0.9}}
           & \underline{51.8\scriptsize{$\pm$4.7}}
           & \underline{53.7\scriptsize{$\pm$0.9}}
           & \underline{51.8\scriptsize{$\pm$4.7}}
           & 50.9\scriptsize{$\pm$2.3}
           & 47.2\scriptsize{$\pm$3.2}
           & 50.3\scriptsize{$\pm$1.0}
           & 34.5\scriptsize{$\pm$4.3}
           & 51.0\scriptsize{$\pm$2.1}
           & 47.0\scriptsize{$\pm$3.7}
           & \textbf{62.9\scriptsize{$\pm$0.2}}
           & \textbf{60.4\scriptsize{$\pm$2.7}}
           & 
           &  \\
\midrule
\multirow{3}{*}{AD}

  & 1\%    & 77.2\scriptsize{$\pm$1.1}
           & 56.3\scriptsize{$\pm$6.1}
           & 72.9\scriptsize{$\pm$1.8}
           & \underline{65.3\scriptsize{$\pm$1.1}}
           & 76.1\scriptsize{$\pm$0.0}
           & 43.2\scriptsize{$\pm$0.0}
           & \underline{78.5\scriptsize{$\pm$0.6}}
           & 61.3\scriptsize{$\pm$3.4}
           & 76.1\scriptsize{$\pm$0.0}
           & 43.2\scriptsize{$\pm$0.0}
           & \underline{78.5\scriptsize{$\pm$0.5}}
           & 61.3\scriptsize{$\pm$2.9}
           & \textbf{81.7\scriptsize{$\pm$0.4}}
           & \textbf{74.9\scriptsize{$\pm$0.7}}
           & \multirow{3}{*}{85.5\scriptsize{$\pm$0.1}}
           & \multirow{3}{*}{78.4\scriptsize{$\pm$0.4}} \\
  & 0.1\%  & 75.7\scriptsize{$\pm$1.3}
           & 56.3\scriptsize{$\pm$5.8}
           & 67.8\scriptsize{$\pm$2.7}
           & \underline{61.4\scriptsize{$\pm$1.0}}
           & 76.0\scriptsize{$\pm$0.1}
           & 43.2\scriptsize{$\pm$0.0}
           & \underline{76.3\scriptsize{$\pm$1.7}}
           & 55.4\scriptsize{$\pm$7.6}
           & 76.1\scriptsize{$\pm$0.0}
           & 43.2\scriptsize{$\pm$0.0}
           & 76.2\scriptsize{$\pm$1.8}
           & 56.0\scriptsize{$\pm$7.6}
           & \textbf{81.5\scriptsize{$\pm$0.6}}
           & \textbf{70.8\scriptsize{$\pm$1.1}}
           & 
           &  \\
  & 0.01\% & 71.3\scriptsize{$\pm$4.1}
           & 51.6\scriptsize{$\pm$6.5}
           & 70.1\scriptsize{$\pm$2.2}
           & \underline{60.8\scriptsize{$\pm$0.8}}
           & 66.5\scriptsize{$\pm$1.9}
           & 60.6\scriptsize{$\pm$1.1}
           & 60.9\scriptsize{$\pm$12.7}
           & 43.8\scriptsize{$\pm$5.2}
           & \underline{76.1\scriptsize{$\pm$0.0}}
           & 43.2\scriptsize{$\pm$0.0}
           & 61.9\scriptsize{$\pm$11.3}
           & 45.0\scriptsize{$\pm$5.0}
           & \textbf{77.7\scriptsize{$\pm$1.1}}
           & \textbf{67.6\scriptsize{$\pm$3.4}}
           & 
           &  \\
\midrule
\multirow{3}{*}{JA}
  & 1\%    & 53.0\scriptsize{$\pm$0.8}
           & 31.2\scriptsize{$\pm$1.5}
           & 53.6\scriptsize{$\pm$1.0}
           & 33.4\scriptsize{$\pm$2.3}
           & 49.3\scriptsize{$\pm$0.7}
           & 24.9\scriptsize{$\pm$1.4}
           & 48.0\scriptsize{$\pm$0.7}
           & 19.9\scriptsize{$\pm$1.3}
           & 42.1\scriptsize{$\pm$5.4}
           & 14.8\scriptsize{$\pm$1.4}
           & \underline{54.7\scriptsize{$\pm$0.9}}
           & \underline{35.8\scriptsize{$\pm$1.5}}
           & \textbf{56.2\scriptsize{$\pm$0.8}}
           & \textbf{36.7\scriptsize{$\pm$1.8}}
           & \multirow{3}{*}{71.6\scriptsize{$\pm$0.2}}
           & \multirow{3}{*}{56.7\scriptsize{$\pm$0.5}} \\
  & 0.1\%  & 50.1\scriptsize{$\pm$1.3}
           & 29.9\scriptsize{$\pm$1.9}
           & \textbf{53.1\scriptsize{$\pm$1.0}}
           & 33.0\scriptsize{$\pm$2.3}
           & 44.1\scriptsize{$\pm$3.2}
           & 23.0\scriptsize{$\pm$1.0}
           & 46.6\scriptsize{$\pm$1.1}
           & 18.4\scriptsize{$\pm$2.7}
           & 46.0\scriptsize{$\pm$0.0}
           & 15.8\scriptsize{$\pm$0.0}
           & 51.3\scriptsize{$\pm$1.3}
           & \underline{34.2\scriptsize{$\pm$1.9}}
           & \underline{52.8\scriptsize{$\pm$1.7}}
           & \textbf{35.3\scriptsize{$\pm$1.8}}
           & 
           &  \\
  & 0.01\% & 41.3\scriptsize{$\pm$3.2}
           & 26.5\scriptsize{$\pm$2.4}
           & \textbf{47.4\scriptsize{$\pm$0.8}}
           & 27.1\scriptsize{$\pm$3.1}
           & 44.1\scriptsize{$\pm$2.6}
           & \underline{27.3\scriptsize{$\pm$1.2}}
           & 36.0\scriptsize{$\pm$11.4}
           & 21.0\scriptsize{$\pm$6.8}
           & 45.7\scriptsize{$\pm$0.0}
           & 15.8\scriptsize{$\pm$0.0}
           & 40.3\scriptsize{$\pm$4.4}
           & 26.3\scriptsize{$\pm$2.8}
           & \underline{45.9\scriptsize{$\pm$2.5}}
           & \textbf{32.2\scriptsize{$\pm$0.8}}
           & 
           &  \\
\midrule
\multirow{3}{*}{DA}

  & 1\%    & 53.1\scriptsize{$\pm$0.5}
           & 25.3\scriptsize{$\pm$1.3}
           & 48.1\scriptsize{$\pm$2.2}
           & \underline{32.3\scriptsize{$\pm$0.8}}
           & 49.4\scriptsize{$\pm$3.9}
           & 28.1\scriptsize{$\pm$3.1}
           & 53.1\scriptsize{$\pm$0.7}
           & 24.9\scriptsize{$\pm$1.2}
           & \underline{53.2\scriptsize{$\pm$0.6}}
           & 24.3\scriptsize{$\pm$0.9}
           & 53.1\scriptsize{$\pm$0.7}
           & 24.8\scriptsize{$\pm$1.0}
           & \textbf{53.5\scriptsize{$\pm$0.4}}
           & \textbf{33.3\scriptsize{$\pm$0.6}}
           & \multirow{3}{*}{58.6\scriptsize{$\pm$0.2}}
           & \multirow{3}{*}{38.7\scriptsize{$\pm$0.5}} \\
  & 0.1\%  & 49.4\scriptsize{$\pm$1.7}
           & 30.0\scriptsize{$\pm$1.9}
           & 46.6\scriptsize{$\pm$2.5}
           & \underline{32.1\scriptsize{$\pm$0.4}}
           & 49.0\scriptsize{$\pm$3.0}
           & 29.5\scriptsize{$\pm$0.8}
           & 49.4\scriptsize{$\pm$1.6}
           & 28.9\scriptsize{$\pm$1.9}
           & \underline{50.4\scriptsize{$\pm$1.5}}
           & 27.9\scriptsize{$\pm$1.7}
           & 49.4\scriptsize{$\pm$1.7}
           & 28.8\scriptsize{$\pm$1.8}
           & \textbf{51.8\scriptsize{$\pm$0.6}}
           & \textbf{34.4\scriptsize{$\pm$0.7}}
           &
           &  \\
  & 0.01\% & \underline{46.3\scriptsize{$\pm$1.8}}
           & 31.9\scriptsize{$\pm$1.5}
           & 39.7\scriptsize{$\pm$2.8}
           & \underline{32.6\scriptsize{$\pm$2.0}}
           & 34.3\scriptsize{$\pm$6.4}
           & 29.1\scriptsize{$\pm$5.3}
           & 41.0\scriptsize{$\pm$3.5}
           & 31.9\scriptsize{$\pm$1.4}
           & 41.5\scriptsize{$\pm$7.5}
           & 29.0\scriptsize{$\pm$3.8}
           & 40.8\scriptsize{$\pm$3.6}
           & 31.8\scriptsize{$\pm$1.6}           
           & \textbf{48.5\scriptsize{$\pm$2.4}}
           & \textbf{35.1\scriptsize{$\pm$1.1}}
           &
           &  \\
\midrule
\multirow{3}{*}{RS}
  & 1\%    & 67.9\scriptsize{$\pm$0.6}
           & 67.9\scriptsize{$\pm$0.6}
           & 67.4\scriptsize{$\pm$0.5}
           & 67.4\scriptsize{$\pm$0.5}
           & 54.3\scriptsize{$\pm$2.4}
           & 52.1\scriptsize{$\pm$3.4}
           & 65.5\scriptsize{$\pm$1.8}
           & 65.2\scriptsize{$\pm$2.3}
           & 50.0\scriptsize{$\pm$0.0}
           & 33.3\scriptsize{$\pm$0.0}
           & \underline{68.6\scriptsize{$\pm$0.4}}
           & \underline{68.6\scriptsize{$\pm$0.5}}
           & \textbf{69.0\scriptsize{$\pm$0.2}}
           & \textbf{69.0\scriptsize{$\pm$0.2}}
           & \multirow{3}{*}{78.3\scriptsize{$\pm$0.2}}
           & \multirow{3}{*}{78.2\scriptsize{$\pm$0.2}} \\
  & 0.1\%  & 63.5\scriptsize{$\pm$3.1}
           & 63.5\scriptsize{$\pm$3.2}
           & \underline{64.3\scriptsize{$\pm$1.3}}
           & \underline{64.2\scriptsize{$\pm$1.3}}
           & 47.7\scriptsize{$\pm$2.7}
           & 45.6\scriptsize{$\pm$3.3}
           & 61.3\scriptsize{$\pm$3.2}
           & 60.4\scriptsize{$\pm$4.0}
           & 50.0\scriptsize{$\pm$0.5}
           & 34.1\scriptsize{$\pm$2.3}
           & 63.7\scriptsize{$\pm$2.7}
           & 63.5\scriptsize{$\pm$2.8}
           & \textbf{65.1\scriptsize{$\pm$0.8}}
           & \textbf{64.9\scriptsize{$\pm$0.9}}
           & 
           &  \\
  & 0.01\% & 55.5\scriptsize{$\pm$6.3}
           & 54.8\scriptsize{$\pm$6.6}
           & \textbf{61.6\scriptsize{$\pm$1.7}}
           & \textbf{61.5\scriptsize{$\pm$1.8}}
           & \underline{61.2\scriptsize{$\pm$2.8}}
           & \underline{60.9\scriptsize{$\pm$3.0}}
           & 54.8\scriptsize{$\pm$5.3}
           & 50.9\scriptsize{$\pm$8.4}
           & 50.0\scriptsize{$\pm$0.5}
           & 33.3\scriptsize{$\pm$0.0}
           & 57.7\scriptsize{$\pm$5.6}
           & 57.6\scriptsize{$\pm$5.6}
           & 60.9\scriptsize{$\pm$1.7}
           & 60.2\scriptsize{$\pm$2.0}
           & 
           &  \\
\midrule

\multirow{3}{*}{EP}
  & 1\%    & \underline{64.4\scriptsize{$\pm$1.1}}
           & \underline{64.3\scriptsize{$\pm$1.2}}
           & 62.2\scriptsize{$\pm$0.6}
           & 62.0\scriptsize{$\pm$0.7}
           & 59.7\scriptsize{$\pm$1.3}
           & 58.4\scriptsize{$\pm$2.3}
           & OOM
           & OOM
           & OOM
           & OOM
           & OOT
           & OOT
           & \textbf{78.9\scriptsize{$\pm$0.4}}
           & \textbf{78.9\scriptsize{$\pm$0.5}}
           & \multirow{3}{*}{86.7\scriptsize{$\pm$0.2}}
           & \multirow{3}{*}{86.7\scriptsize{$\pm$0.2}} \\
  & 0.1\%  & \underline{60.1\scriptsize{$\pm$0.9}}
           & \underline{59.8\scriptsize{$\pm$1.0}}
           & 59.0\scriptsize{$\pm$0.8}
           & 58.5\scriptsize{$\pm$1.0}
           & 56.7\scriptsize{$\pm$0.8}
           & 55.8\scriptsize{$\pm$1.3}
           & OOM
           & OOM
           & OOM
           & OOM
           & OOT
           & OOT
           & \textbf{75.2\scriptsize{$\pm$1.5}}
           & \textbf{74.8\scriptsize{$\pm$1.8}}
           & 
           &  \\
  & 0.01\% & \underline{53.6\scriptsize{$\pm$0.9}}
           & \underline{53.2\scriptsize{$\pm$0.8}}
           & 52.7\scriptsize{$\pm$0.7}
           & 51.8\scriptsize{$\pm$1.1}
           & 52.7\scriptsize{$\pm$0.5}
           & 51.7\scriptsize{$\pm$1.2}
           & 50.0\scriptsize{$\pm$0.3}
           & 37.9\scriptsize{$\pm$5.8}
           & OOT
           & OOT
           & OOT
           & OOT
           & \textbf{71.0\scriptsize{$\pm$1.8}}
           & \textbf{70.7\scriptsize{$\pm$2.0}}
           & 
           &  \\
\midrule
\multirow{3}{*}{AI}

  & 1\%    & 55.5\scriptsize{$\pm$0.4}
           & 45.1\scriptsize{$\pm$3.1}
           & 51.2\scriptsize{$\pm$0.3}
           & \underline{51.1\scriptsize{$\pm$0.4}}
           & 55.4\scriptsize{$\pm$0.1}
           & 36.0\scriptsize{$\pm$0.3}
           &\underline{55.9\scriptsize{$\pm$0.4}}
           & 48.4\scriptsize{$\pm$3.4}
           & OOT
           & OOT
           & \underline{55.9\scriptsize{$\pm$0.5}}
           & 48.5\scriptsize{$\pm$3.3}
           & \textbf{61.1\scriptsize{$\pm$0.3}}
           & \textbf{59.8\scriptsize{$\pm$0.2}}
           & \multirow{3}{*}{64.6\scriptsize{$\pm$0.1}}
           & \multirow{3}{*}{62.7\scriptsize{$\pm$0.3}} \\
  & 0.1\%  & 54.3\scriptsize{$\pm$0.6}
           & 47.7\scriptsize{$\pm$1.9}
           & 51.2\scriptsize{$\pm$0.4}
           & 51.1\scriptsize{$\pm$0.6}
           & \underline{55.2\scriptsize{$\pm$0.1}}
           & 36.2\scriptsize{$\pm$0.3}
           & 54.4\scriptsize{$\pm$1.0}
           & \underline{51.5\scriptsize{$\pm$1.9}}
           & 54.0\scriptsize{$\pm$1.1}
           & 48.0\scriptsize{$\pm$4.1}
           & 54.3\scriptsize{$\pm$1.0}
           & 51.4\scriptsize{$\pm$2.0}
           & \textbf{61.5\scriptsize{$\pm$0.4}}
           & \textbf{58.4\scriptsize{$\pm$0.7}}
           & 
           &  \\
  & 0.01\% & 51.7\scriptsize{$\pm$1.5}
           & 49.7\scriptsize{$\pm$1.9}
           & 52.3\scriptsize{$\pm$0.9}
           & 50.5\scriptsize{$\pm$0.5}
           & 53.5\scriptsize{$\pm$0.4}
           & 43.0\scriptsize{$\pm$4.4}
           & \underline{53.8\scriptsize{$\pm$1.3}}
           & \underline{50.6\scriptsize{$\pm$1.5}}
           & 51.4\scriptsize{$\pm$2.0}
           & 48.7\scriptsize{$\pm$2.7}
           & \underline{53.8\scriptsize{$\pm$1.3}}
           & 50.5\scriptsize{$\pm$1.6}
           & \textbf{61.1\scriptsize{$\pm$0.5}}
           & \textbf{57.5\scriptsize{$\pm$1.1}}
           & 
           &  \\
\midrule
\multirow{3}{*}{CO}

  & 1\%    & 70.9\scriptsize{$\pm$0.6}
           & 39.2\scriptsize{$\pm$0.6}
           & 67.3\scriptsize{$\pm$0.3}
           & 37.2\scriptsize{$\pm$0.7}
           & 65.6\scriptsize{$\pm$1.2}
           & 28.4\scriptsize{$\pm$0.7}
           & 62.4\scriptsize{$\pm$1.3}
           & 35.9\scriptsize{$\pm$1.2}
           & OOT
           & OOT
           & \underline{72.7\scriptsize{$\pm$0.2}}
           & \underline{41.8\scriptsize{$\pm$1.0}}
           & \textbf{73.4\scriptsize{$\pm$0.2}}
           & \textbf{51.6\scriptsize{$\pm$2.6}}
           & \multirow{3}{*}{87.6\scriptsize{$\pm$0.1}}
           & \multirow{3}{*}{81.3\scriptsize{$\pm$0.2}} \\
  & 0.1\%  & 63.9\scriptsize{$\pm$1.1}
           & 32.4\scriptsize{$\pm$2.1}
           & 64.3\scriptsize{$\pm$0.7}
           & 32.4\scriptsize{$\pm$1.4}
           & 56.3\scriptsize{$\pm$0.3}
           & 19.3\scriptsize{$\pm$0.9}
           & 61.5\scriptsize{$\pm$1.1}
           & 33.0\scriptsize{$\pm$1.8}
           & 47.9\scriptsize{$\pm$2.6}
           & 9.7\scriptsize{$\pm$0.8}
           & \underline{66.7\scriptsize{$\pm$0.8}}
           & \underline{38.3\scriptsize{$\pm$1.7}}
           & \textbf{69.6\scriptsize{$\pm$0.6}}
           & \textbf{46.4\scriptsize{$\pm$1.3}}
           & 
           &  \\
  & 0.01\% & 52.3\scriptsize{$\pm$2.1}
           & 27.0\scriptsize{$\pm$2.7}
           & \underline{54.3\scriptsize{$\pm$1.3}}
           & \underline{36.8\scriptsize{$\pm$2.0}}
           & 44.6\scriptsize{$\pm$1.0}
           & 29.4\scriptsize{$\pm$2.3}
           & 52.1\scriptsize{$\pm$3.3}
           & 24.5\scriptsize{$\pm$4.0}
           & 46.9\scriptsize{$\pm$3.9}
           & 9.8\scriptsize{$\pm$1.4}
           & 52.8\scriptsize{$\pm$2.3}
           & 26.8\scriptsize{$\pm$4.0}
           & \textbf{60.9\scriptsize{$\pm$0.9}}
           & \textbf{45.4\scriptsize{$\pm$1.7}}
           & 
           &  \\
\midrule

\multirow{3}{*}{HI}
  & 1\%    & 52.4\scriptsize{$\pm$0.4}
           & 52.3\scriptsize{$\pm$0.4}
           & 52.5\scriptsize{$\pm$0.4}
           & 48.8\scriptsize{$\pm$1.1}
           & 52.2\scriptsize{$\pm$0.3}
           & 44.1\scriptsize{$\pm$2.3}
           & 50.3\scriptsize{$\pm$0.4}
           & 48.6\scriptsize{$\pm$2.0}
           & OOT
           & OOT
           & \underline{53.7\scriptsize{$\pm$0.4}}
           & \underline{53.4\scriptsize{$\pm$0.4}}
           & \textbf{54.6\scriptsize{$\pm$0.4}}
           & \textbf{54.5\scriptsize{$\pm$0.5}}
           & \multirow{3}{*}{75.2\scriptsize{$\pm$0.0}}
           & \multirow{3}{*}{75.2\scriptsize{$\pm$0.1}} \\
  & 0.1\%  & 50.8\scriptsize{$\pm$0.4}
           & 50.7\scriptsize{$\pm$0.4}
           & 51.1\scriptsize{$\pm$0.5}
           & 48.7\scriptsize{$\pm$0.7}
           & 51.4\scriptsize{$\pm$0.4}
           & 45.4\scriptsize{$\pm$2.5}
           & 50.6\scriptsize{$\pm$0.4}
           & 47.7\scriptsize{$\pm$2.0}
           & 50.0\scriptsize{$\pm$0.0}
           & 33.3\scriptsize{$\pm$0.0}
           & \underline{51.4\scriptsize{$\pm$0.3}}
           & \underline{51.3\scriptsize{$\pm$0.3}}
           & \textbf{52.2\scriptsize{$\pm$0.5}}
           & \textbf{52.1\scriptsize{$\pm$0.5}}
           & 
           &  \\
  & 0.01\% & 50.4\scriptsize{$\pm$0.5}
           & 50.3\scriptsize{$\pm$0.5}
           & 50.4\scriptsize{$\pm$0.7}
           & 49.2\scriptsize{$\pm$1.3}
           & 50.6\scriptsize{$\pm$0.3}
           & 39.6\scriptsize{$\pm$2.4}
           & 50.3\scriptsize{$\pm$0.7}
           & 47.3\scriptsize{$\pm$3.6}
           & 50.0\scriptsize{$\pm$0.0}
           & 33.3\scriptsize{$\pm$0.0}
           & \underline{50.8\scriptsize{$\pm$0.6}}
           & \underline{50.7\scriptsize{$\pm$0.6}}
           & \textbf{51.9\scriptsize{$\pm$0.2}}
           & \textbf{51.4\scriptsize{$\pm$0.2}}
           & 
           &  \\
\midrule
\multirow{3}{*}{MI}
  & 1\%    & \underline{54.1\scriptsize{$\pm$0.2}}
           & 21.4\scriptsize{$\pm$0.8}
           & 53.3\scriptsize{$\pm$0.5}
           & \underline{22.2\scriptsize{$\pm$0.1}}
           & 52.7\scriptsize{$\pm$0.2}
           & 16.8\scriptsize{$\pm$0.8}
           & 50.2\scriptsize{$\pm$4.7}
           & 20.8\scriptsize{$\pm$1.6}
           & OOT
           & OOT
           & OOM
           & OOM
           & \textbf{54.4\scriptsize{$\pm$0.0}}
           & \textbf{22.7\scriptsize{$\pm$0.1}}
           & \multirow{3}{*}{57.1\scriptsize{$\pm$0.1}}
           & \multirow{3}{*}{28.9\scriptsize{$\pm$0.5}} \\
  & 0.1\%  & 52.8\scriptsize{$\pm$0.7}
           & \underline{20.1\scriptsize{$\pm$2.7}}
           & 52.5\scriptsize{$\pm$0.8}
           & 19.8\scriptsize{$\pm$2.8}
           & 51.2\scriptsize{$\pm$1.5}
           & 16.4\scriptsize{$\pm$2.0}
           & 52.2\scriptsize{$\pm$1.9}
           & 17.3\scriptsize{$\pm$3.0}
           & OOT
           & OOT
           & \underline{53.4\scriptsize{$\pm$0.6}}
           & 19.9\scriptsize{$\pm$2.3}
           & \textbf{53.6\scriptsize{$\pm$0.1}}
           & \textbf{21.5\scriptsize{$\pm$0.4}}
           & 
           &  \\
  & 0.01\% & 49.8\scriptsize{$\pm$1.6}
           & \underline{20.2\scriptsize{$\pm$2.7}}
           & 50.2\scriptsize{$\pm$1.1}
           & 19.6\scriptsize{$\pm$3.9}
           & 48.2\scriptsize{$\pm$3.3}
           & 15.7\scriptsize{$\pm$1.5}
           & 51.9\scriptsize{$\pm$0.2}
           & 13.8\scriptsize{$\pm$0.3}
           & \underline{52.0\scriptsize{$\pm$0.0}}
           & 13.7\scriptsize{$\pm$0.0}
           & 50.7\scriptsize{$\pm$1.2}
           & 19.8\scriptsize{$\pm$1.5}
           & \textbf{53.2\scriptsize{$\pm$0.3}}
           & \textbf{20.6\scriptsize{$\pm$0.5}}
           & 
           &  \\
\bottomrule
\end{tabular}%
}
\end{table*}
\subsection{Overall Performance} 
\label{exp_effect}

To validate the effectiveness of $\text{C}^{2}\text{TC}$, we compare it against six baselines on the classification task across condensation ratios $r \in \{1\%, 0.1\%, 0.01\%\}$, as reported in Table~\ref{tab:class}.
We make the following key observations: 
(1) Overall, $\text{C}^{2}\text{TC}$ demonstrates superior performance compared to two categories of baselines. 
It achieves state-of-the-art accuracy on 8 out of 10 datasets and best Macro-F1 score in 29 out of 30 experimental cases. 
On most datasets, models trained on our condensed data achieve accuracy comparable to full-data training. 
When we reduce the size of the tabular data by 99.0\%, $\text{C}^{2}\text{TC}$ can approximate the original test accuracy by over 86\% on 7 out of the 10 datasets.
\red{(2)} Our method typically delivers stronger class-balanced performance, with especially pronounced gains on imbalanced datasets. For example, 
at $r=1\%$, it improves Macro-F1 over the best baseline by 9.6\% on AD and 9.8\% on CO.
These results indicate that our class-adaptive condensation algorithm can better preserves minority-class performance, thereby offering higher utility.
(3) While state-of-the-art DC methods such as GM and MTT generally outperform traditional coreset baselines, they often suffer from OOM or OOT failures on large datasets (e.g., EP and MI). In contrast, our method typically attains state-of-the-art performance and remains scalable across all datasets, offering a more favorable efficiency–effectiveness trade-off.

\subsection{Cross-architecture Generalization}
\label{exp_general}
To evaluate the generalization of $\text{C}^{2}\text{TC}$, Table~\ref{tab:cross-arch} reports results across six benchmark architectures for deep tabular learning: three MLP-based models (MLP, MLP\_PLR~\cite{gorishniy2022embeddings}, and RealMLP~\cite{holzmller2024better}) and three Transformer-based models (FT\_Transformer (FT\_T)~\cite{gorishniy2021revisiting}, TabNet~\cite{arik2021tabnet}, and TabR~\cite{gorishniy2024tabr}). 
For comparison, we also report results of Random, DM, GM, and MTT.  
On the high-dimensional dataset EP, existing DC methods fail due to OOM or OOT issues, so results are reported only for Random and $\text{C}^{2}\text{TC}$, with $\text{C}^{2}\text{TC}$ outperforming Random by approximately 6\% in terms of average accuracy. 
On MI, GM is also omitted due to OOT.  
Overall, $\text{C}^{2}\text{TC}$ achieves the highest average accuracy across datasets, consistently benefiting a variety of deep tabular models. 
For instance, compared to MTT, the strongest competing DC baseline, $\text{C}^{2}\text{TC}$ yields an average improvement of 3.1\% on AD and more than 8\% on EL.  
These results confirm the superior generalization ability of $\text{C}^{2}\text{TC}$, which stems from our theoretically grounded objective that ensures the quality of the condensed datasets, together with the model-free design of our algorithm that avoids architecture-specific biases and thereby enhances generalization across diverse models.
\begin{table}[t]
\caption{Generalization of DC methods across deep tabular learning models. Avg. stands for the
average test accuracy of MLP, MLP\_PLR, RealMLP, FT\_T, TabNet, and TabR. Results for DM, GM, and MTT on EP, and GM on MI, are omitted as they fail to scale to these datasets due to time or memory constraints. All results are reported in terms of accuracy.}
\label{tab:cross-arch}
\centering
\scriptsize
\resizebox{\linewidth}{!}
{\begin{tabular}{llccc ccc c}
\toprule
\multirow{2}{*}{\textbf{Dataset}} 
& \multirow{2}{*}{\textbf{Method}} 
& \multicolumn{3}{c}{\textbf{MLP-Based}} 
& \multicolumn{3}{c}{\textbf{Transformer-Based}} 
& \multirow{2}{*}{\textbf{Avg.}}  \\
\cmidrule(lr){3-5} \cmidrule(lr){6-8}
& & MLP & MLP\_PLR & RealMLP & FT\_T & TabNet & TabR & \\
\midrule
\multirow{5}{*}{EL}
  & Random & 53.5 & 55.8 & 52.3 & 53.9 & 50.3 & 52.2 & 53.0  \\
  & DM     & 51.6 & 55.1 & 55.1 & 57.2 & 49.3 & 53.0 & 53.6  \\
  & GM     & 50.3 & 50.6 & 50.3 & 50.6 & 50.3 & 50.6 & 50.5  \\
  & MTT    & 52.6 & 57.0 & 52.7 & 56.0 & \textbf{51.1} & 53.9 & 53.9 \\
  & $\text{C}^{2}\text{TC}$   & \textbf{64.5} & \textbf{67.5} & \textbf{62.8} & \textbf{64.5} & 51.0 & \textbf{61.3} & \textbf{62.0}  \\

\midrule
\multirow{5}{*}{AD}
  & Random & 75.7 & 77.1 & 73.6 & 76.2 & 64.3 & 76.3 & 73.9  \\
  & DM     & 76.3 & 76.0 & 73.8 & 77.2 & 71.6 & 77.6 & 75.4  \\
  & GM     & 76.1 & 75.9 & 68.7 & 76.6 & 69.8 & 76.1 & 73.9  \\
  & MTT    & 76.2 & 77.3 & 73.9 & 76.8 & 71.1 & 77.7 & 75.5  \\
  & $\text{C}^{2}\text{TC}$   & \textbf{81.5} & \textbf{80.3} & \textbf{77.8} & \textbf{79.3} & \textbf{72.5} & \textbf{79.9} & \textbf{78.6}  \\

\midrule
\multirow{2}{*}{EP}
  & Random & 60.1 & 66.8 & \textbf{50.4} & 50.0 & 50.1 & 53.8 & 55.2  \\
  & $\text{C}^{2}\text{TC}$   & \textbf{75.2} & \textbf{83.7} & 50.2 & \textbf{50.3} & \textbf{50.5} &\textbf{56.8} & \textbf{61.1}  \\
\midrule
\multirow{5}{*}{CO}
  & Random & 63.9 & 67.1 & 52.5 & 61.6 & 56.6 & 69.2 & 61.8  \\
  & DM     & 61.5 & 56.9 & \textbf{54.3} & 57.7 & 57.3 & 60.0 & 57.9  \\
  & GM     & 47.9 & 51.9 & 43.8 & 53.1 & 54.9 & 45.6 & 49.5  \\
  & MTT    & 66.7 & 68.1 & 53.5 & 63.6 & 60.1 & \textbf{69.7} & 63.6  \\
  & $\text{C}^{2}\text{TC}$   & \textbf{69.6} & \textbf{70.3} & 53.9 & \textbf{64.7} & \textbf{61.5} & 68.5 & \textbf{64.8}  \\

\midrule
\multirow{4}{*}{MI}
  & Random & 52.8 & 52.8 & 52.0 & 51.8 & 46.2 & \textbf{52.7} & 51.4  \\
  & DM     & 52.2 & \textbf{53.2} & 51.7 & 52.0 & 50.9 & 46.3 & 51.1  \\
  & MTT    & 53.4 & 52.0 & 52.1 & \textbf{52.1} & 49.8 & 52.5 & 52.0  \\
  & $\text{C}^{2}\text{TC}$   & \textbf{53.6} & 53.1 & \textbf{52.2} & 52.0 & \textbf{51.5} & 52.2 & \textbf{52.4}  \\
\bottomrule
\end{tabular}}
\end{table}

\subsection {Efficiency Evaluation}
\label{exp_efficiency}

\label{effi_app}
\noindent\textbf{Condensation time over all datasets.}
In Figure~\ref{fig:fig1}, we first compare the condensation time of our approach with state-of-the-art DC
baselines across all tabular datasets. 
As can be seen, $\text{C}^{2}\text{TC}$ exhibits significantly lower condensation time compared to existing DC methods, consistently achieving at least 2 orders of magnitude speedup across all datasets.
For example, on the largest dataset MI, $\text{C}^{2}\text{TC}$ only takes 39.2 seconds during condensation, while MTT requires over 78 hours to complete the synthesis process and GM fails to finish within the time limit.
While on the high-dimensional dataset (EP), existing DC methods either suffer from OOT or OOM issues and are therefore not included in the figure, whereas $\text{C}^{2}\text{TC}$ remains scalable, completing condensation in under 72 seconds.
This remarkable efficiency gain can be attributed to our training-free design, which avoids the Siamese architecture and gradient-based iterative updates adopted by prior methods, thereby eliminating major computational bottlenecks.
\begin{figure} [t]
    \centering
    \includegraphics[width=1\linewidth]{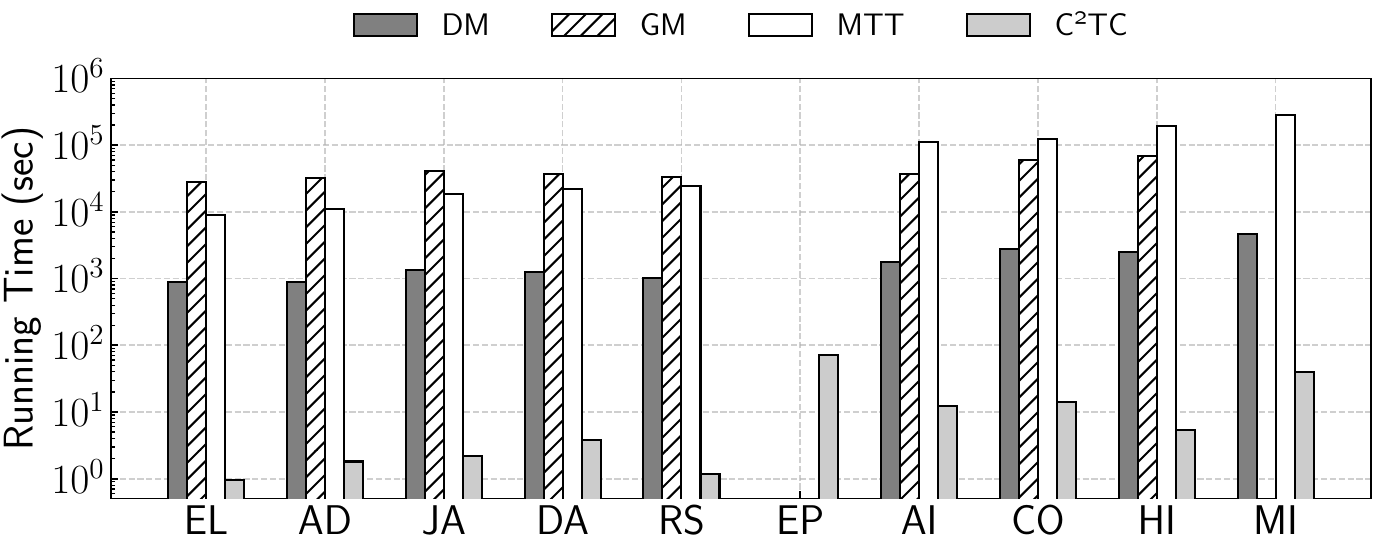}
     \caption{Condensation time comparison.}
    \label{fig:fig1}
\end{figure}

\noindent\textbf{Condensation time by varying $N$ and $F$.}
In this experiment, we study the scalability of condensation methods with respect to sample size $N$ and feature dimension $F$ on the MI and EP datasets, respectively. 
Specifically, on MI we vary the sample size from 20\% to 80\% to form new datasets under the default reduction rate (0.1\%), 
whereas on EP we fix the sample size and vary the feature dimension from 20\% to 80\% under a reduction rate of 0.01\%.
Figure~\ref{fig:fig3}(a) and Figure~\ref{fig:fig3}(b) show the condensation time of $\text{C}^{2}\text{TC}$ and the three DC methods with varying $N$ and $F$, respectively.
As observed, $\text{C}^{2}\text{TC}$ outperforms the other algorithms by a significant margin and scales well. 
For example, when the percentage of $N$ is 20\% in Figure~\ref{fig:fig3}(a), 
the condensation times of DM, GM, and MTT are 2340.2 seconds, $6.0 \times 10^{4}$ seconds, and $4.3 \times 10^{4}$ seconds, respectively, 
whereas $\text{C}^{2}\text{TC}$ requires only $15.2$ seconds. 
When the percentage of $N$ increases to 80\%, $\text{C}^{2}\text{TC}$ completes in $38.1$ seconds, while DM, GM, and MTT scale up to 3943.7, $1.1 \times 10^{5}$, and $2.3 \times 10^{5}$ seconds, respectively.  
In Figure~\ref{fig:fig3}(b), we find that $\text{C}^{2}\text{TC}$ consistently achieves at least 2 orders of magnitude speedup over existing DC methods across different ratios of $F$ on the EP dataset. 
As $F$ increases from 20\% to 80\%, the condensation time of $\text{C}^{2}\text{TC}$ rises only slightly from $4.6$ seconds to $14.4$ seconds. 
In contrast, the runtime of the fastest DC baseline, DM, grows rapidly from 1920.7 seconds to 6749.5 seconds.  
These results highlight the superior scalability of $\text{C}^{2}\text{TC}$, in line with our complexity analysis.

\begin{figure}[t] 
  \centering
  \includegraphics[width=0.7\linewidth]{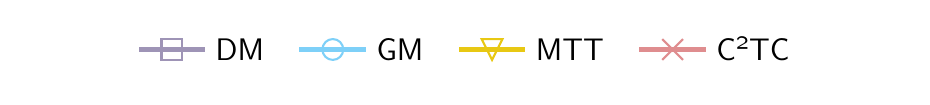}
  \begin{subfigure}[b]{0.48\columnwidth} 
    \centering

    \includegraphics[width=\linewidth]{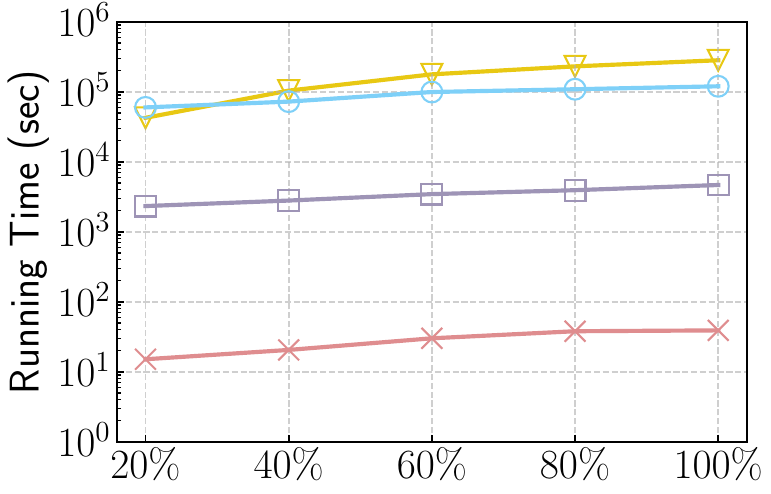}
    \caption{MI (Vary \#samples $N$)}

  \end{subfigure}
    \hfill
    \begin{subfigure}[b]{0.48\columnwidth}
    \centering

    \includegraphics[width=\linewidth]{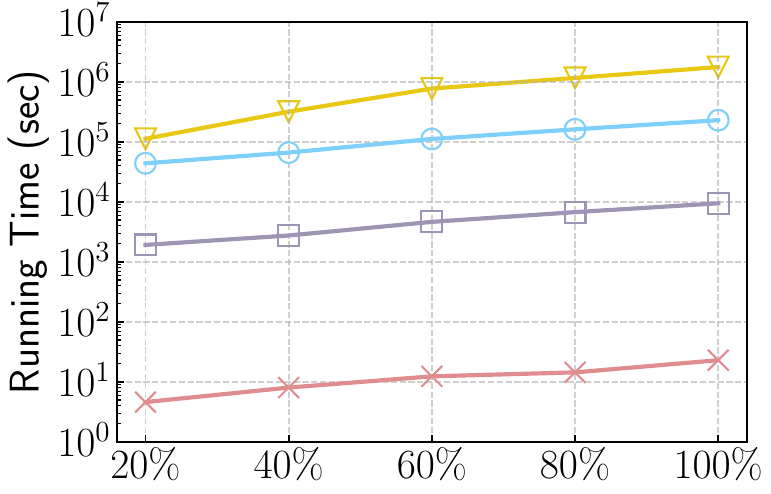}
    \caption{EP (Vary \#features $F$)}

  \end{subfigure}

  \caption{Condensation time by varying $N$ and $F$.}
  \label{fig:fig3}
\end{figure}

\begin{figure}[t] 
  \centering
  \includegraphics[width=0.7\linewidth]{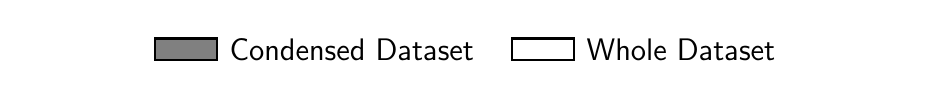}
  \begin{subfigure}[b]{0.32\columnwidth} 
    \centering
    \includegraphics[width=\linewidth]{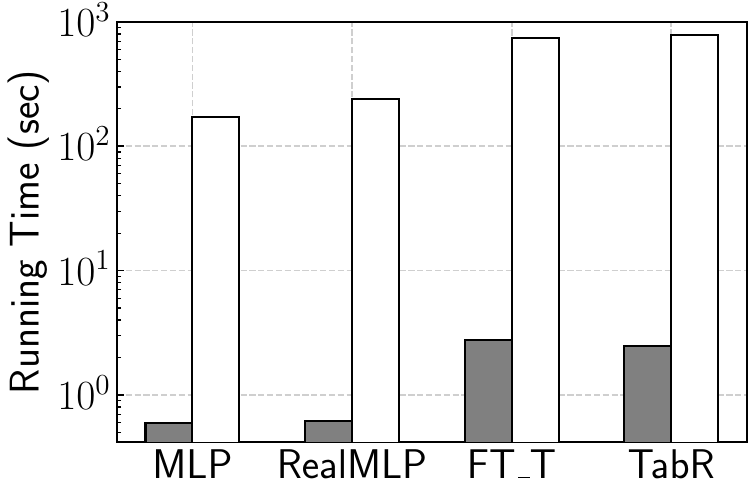}
    \caption{EL}
  \end{subfigure}
    \hfill
    \begin{subfigure}[b]{0.32\columnwidth} 
    \centering
    \includegraphics[width=\linewidth]{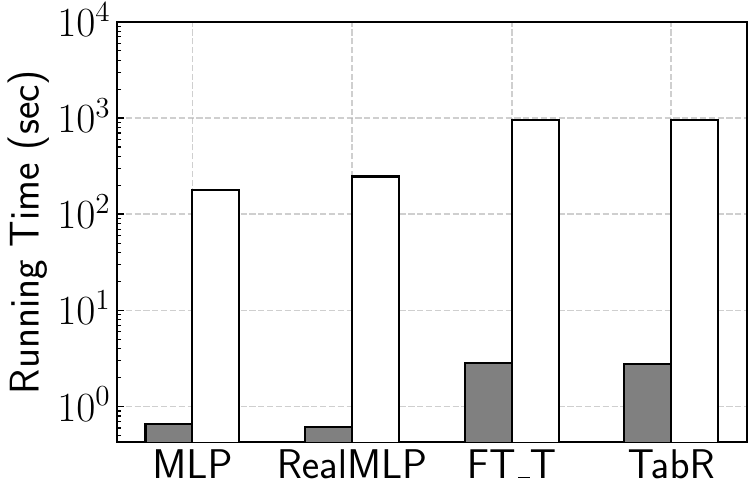}
    \caption{AD}
  \end{subfigure}
  \hfill
  \begin{subfigure}[b]{0.32\columnwidth} 
    \centering
    \includegraphics[width=\linewidth]{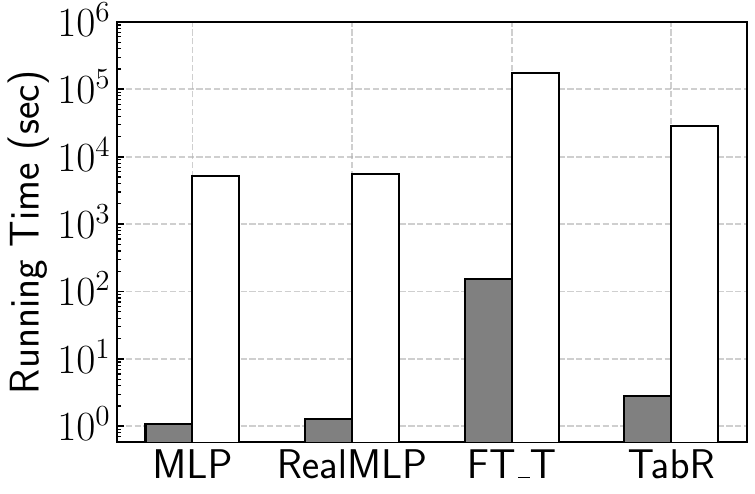}
    \caption{EP}
    
    \label{fig:svg1}
  \end{subfigure}

    \vspace{1mm}
  \hfill
  \begin{subfigure}[b]{0.32\columnwidth} 
    \centering
     
    \includegraphics[width=\linewidth]{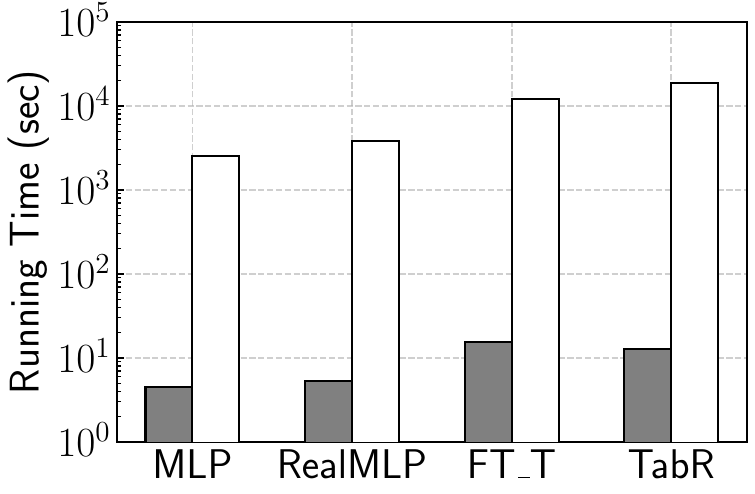}
    \caption{CO}
  \end{subfigure}  
  \hfill
  \begin{subfigure}[b]{0.32\columnwidth} 
    \centering
    \includegraphics[width=\linewidth]{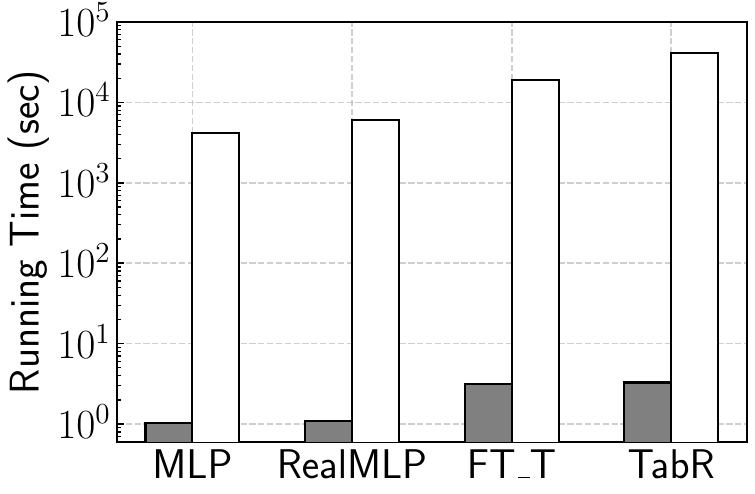}
    \caption{HI}
  \end{subfigure}
  \hfill
  \begin{subfigure}[b]{0.32\columnwidth} 
    \centering
    \includegraphics[width=\linewidth]{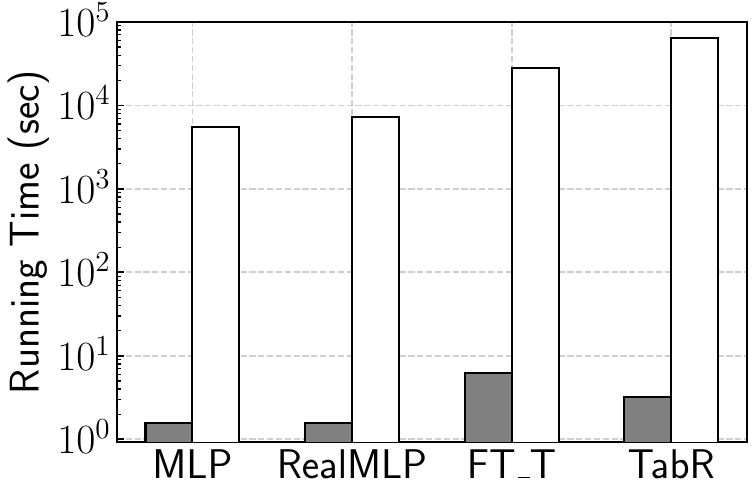}
    \caption{MI}
  \end{subfigure}

  \caption{\myblue{Training time on condensed vs. whole datasets.}}
  \label{fig:fig2}
\end{figure}

\noindent\textbf{Model training time on condensed data.} We further compare the training time of deep models on the condensed datasets against their original counterparts. Specifically, we evaluate four representative tabular learning architectures, namely MLP, RealMLP, FT-Transformer (FT-T), and TabR, on six datasets with different scales (EL, AD, EP, CO, HI, and MI).
As shown in Figure~\ref{fig:fig2}, the training time on the condensed datasets is significantly reduced compared with that on the original datasets in all cases.
For example, on the HI dataset, training RealMLP on the original data takes over 5993 seconds, whereas training on the small-scale synthetic data requires only about 1 second.
On the largest MI dataset, training TabR on the condensed dataset is accelerated by over 4 orders of magnitude compared with training on the original dataset,
showing the efficiency and practicality of tabular dataset condensation.

\subsection {Ablation Study}
\label{exp_ablation}

\begin{table}[t]
\caption{Comparison of label allocation strategies by Accuracy (Acc) and Macro-F1 (MF1).}
\label{tab:k_ablation_combined}
\centering
\scriptsize
\resizebox{\linewidth}{!}
{\begin{tabular}{l cc cc cc cc cc}
\toprule
\multirow{2}{*}{\textbf{Strategy}} 
& \multicolumn{2}{c}{\textbf{EL}} 
& \multicolumn{2}{c}{\textbf{AD}} 
& \multicolumn{2}{c}{\textbf{EP}} 
& \multicolumn{2}{c}{\textbf{CO}} 
& \multicolumn{2}{c}{\textbf{MI}} \\
\cmidrule(lr){2-3} \cmidrule(lr){4-5} \cmidrule(lr){6-7} \cmidrule(lr){8-9} \cmidrule(lr){10-11}
& Acc & MF1 & Acc & MF1 & Acc & MF1 & Acc & MF1 & Acc & MF1 \\
\midrule
Ratio     & 63.0 & 62.9 & 80.4 & 69.7 & 69.9 & 69.0 & 65.6 & 33.1 & 53.5 & 20.9 \\
FIPC       & 63.0 & 62.9 & 74.2 & 70.7 & 69.9 & 69.0 & 47.3 & 40.5 & 34.2 & 20.8 \\
$\text{C}^{2}\text{TC}$ & \textbf{64.5} & \textbf{63.5} & \textbf{81.5} & \textbf{70.8} & \textbf{75.2} & \textbf{74.8} & \textbf{69.6} & \textbf{46.4} & \textbf{53.6} & \textbf{21.5} \\
\bottomrule
\end{tabular}}
\end{table}

\noindent\textbf{Analysis of label allocation strategies.}
To validate the effectiveness of our adaptive allocation strategy, which is developed on top of our proposed training-free objective to appropriately determine the per-class cluster number \revise{${n'_i}$}, we replace it with two static label allocation strategies, Ratio and FIPC.
Ratio allocates clusters proportionally to the original class distribution, whereas FIPC assigns an equal number of synthetic samples to each class.
\revise{As shown in Table~\ref{tab:k_ablation_combined}, our class-adaptive strategy consistently achieves superior performance on both metrics across all datasets.
For instance, it improves accuracy by 5.3\% and Macro-F1 by 5.8\% on EP, and by 4.0\% and 5.9\% on CO, demonstrating strong robustness across different class distribution scenarios. This is because our $\text{C}^{2}\text{TC}$ dynamically balances the influence of majority and minority classes, leading to a better trade-off between overall accuracy and per-class performance.
Compared with FIPC, Ratio achieves higher accuracy but lower Macro-F1 on imbalanced datasets (AD, CO, MI), as it allocates more clusters to majority classes and overlooks minority ones. FIPC, on the other hand, enforces uniform allocation across classes, improving minority-class performance at the cost of overall accuracy dominated by majority classes.}




\begin{table}[t]
\caption{\red{Ablation study of $\text{C}^{2}\text{TC}$ modules.}}
\label{tab:ablation}
\centering
\scriptsize
\setlength{\tabcolsep}{4pt}
\renewcommand{\arraystretch}{1.1}
\resizebox{\linewidth}{!}{
\begin{tabular}{lccccccc}
\toprule
\textbf{Modules} & Metrics  & AD & JA & DA & AI & CO & MI \\
\midrule
w/o Soft Step Size  & Acc    & 74.5 & 50.9 & 36.8 & --   & 60.9 & 46.6  \\
w/o Soft Allocation & Acc    & --   & 51.7 & 50.5 & --   & 58.3 & 49.8  \\
$\text{C}^{2}\text{TC}$                & Acc  & 81.5 & 52.8 & 51.8 & 61.5 & 69.6 & 53.6  \\
\midrule
w/o Scale Factor       & Macro-F1    & 69.0 & 29.2 & 32.6 & --   & 35.8 & 17.0 \\
$\text{C}^{2}\text{TC}$                & Macro-F1  & 70.8 & 35.3 & 34.4 & 58.4 & 46.4 & 21.5 \\
\bottomrule
\end{tabular}}
\end{table}



\noindent\textbf{\red{Analysis of key modules}.}
We analyze the impact of three key components of HFILS on label allocation, as shown in Table~\ref{tab:ablation}.
w/o Scale Factor removes the class-wise scale factor $\frac{1}{n_i^{\gamma}}$ from the objective in Eq.~(\ref{objective}), which is equivalent to setting $\gamma = 0$.
w/o Soft Step Size replaces the randomized step sampling $s \sim \text{Uniform}({1,2,\ldots,s_{\max}})$ 
with a fixed step size $s_{\max}$.
w/o Soft Allocation modifies the redistribution rule 
by reallocating reduced clusters to a single target class instead of distributing them across multiple classes.
w/o Scale Factor and w/o Soft Step Size primarily contribute to imbalanced datasets (AD, JA, DA, CO, MI),
while w/o Soft Allocation can be applied to multi-class datasets (JA, DA, CO, MI).
It can be observed that removing the Scale Factor
results in a notable drop in Macro-F1, e.g., from 35.3\% to 29.2\% on JA and from 46.4\% to 35.8\% on CO.
The degradation stems from the absence of a balancing mechanism, allowing majority classes to dominate the allocation and thereby undermine the contribution of minority classes, which ultimately leads to lower average class-wise performance.
Removing randomness from the allocation search (w/o Soft Step Size) causes the accuracy to drop from 51.8\% to 36.8\% on DA and from 69.6\% to 60.9\% on CO, highlighting the importance of stochastic step-size sampling in escaping local optima and improving allocation robustness.
w/o Soft Allocation reduces accuracy from 69.6\% to 58.3\% on CO and from 53.6\% to 49.8\% on MI. This is because soft allocation promotes diverse redistribution and stabilizes the search process, both of which are crucial for effective allocation.

\myblue{Our ablation study on \red{categorical feature encoding} can be found in Appendix~\ref{ab_app}~\cite{appendix}.}

\begin{figure}[t] 
  \centering
  \includegraphics[width=0.7\linewidth]{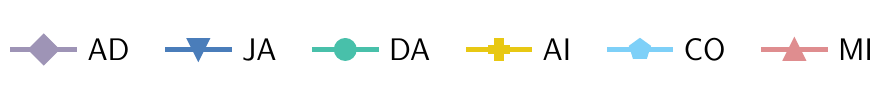}
  \begin{subfigure}[b]{0.49\columnwidth} 
    \centering
    \includegraphics[width=0.9\linewidth]{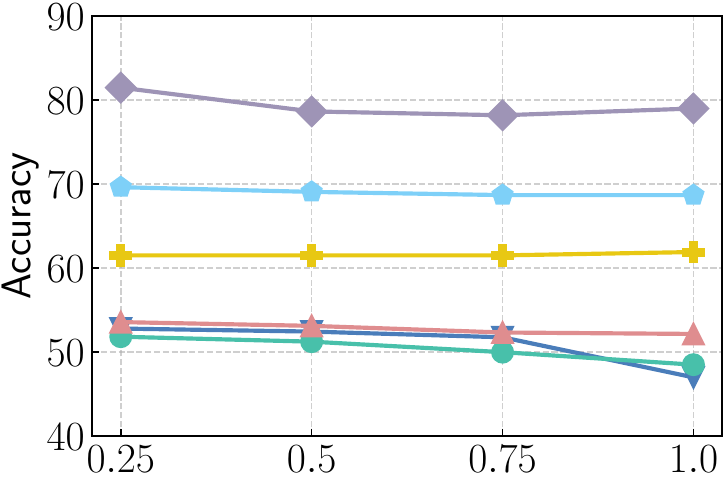}
    \caption{$\gamma$ Accuracy}
  \end{subfigure}
  \begin{subfigure}[b]{0.49\columnwidth} 
    \centering
    \includegraphics[width=0.9\linewidth]{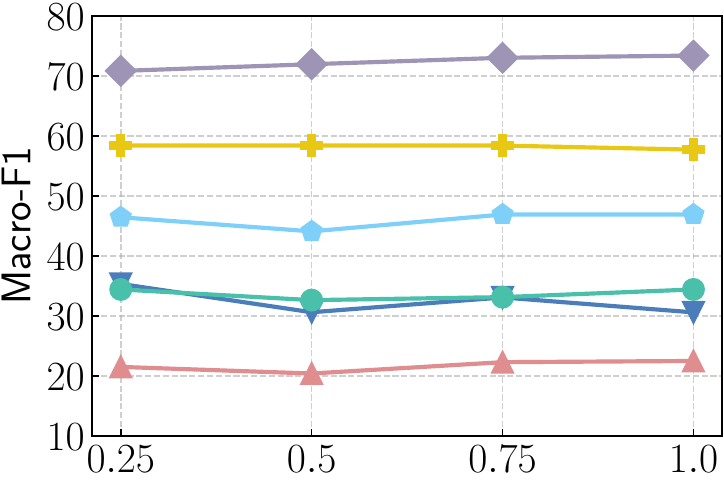}
    \caption{$\gamma$ Macro-F1}
    
    \label{fig:svg1}
  \end{subfigure}

    \vspace{1mm}
  \begin{subfigure}[b]{0.49\columnwidth} 
    \centering
     
    \includegraphics[width=0.9\linewidth]{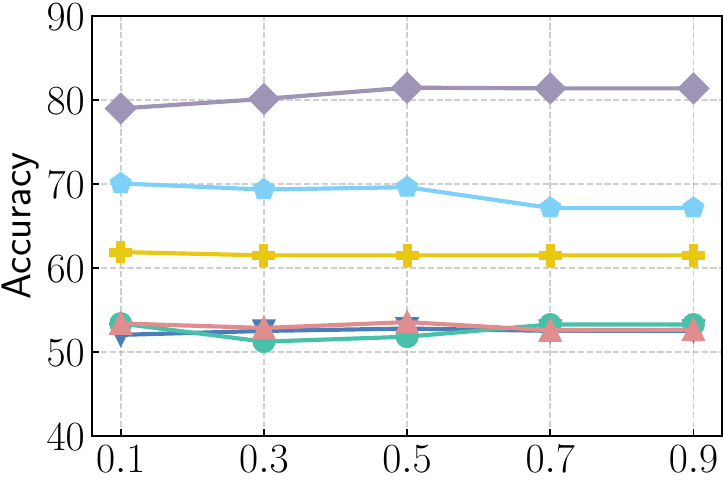}
    \caption{$l$ Accuracy}
  \end{subfigure}  
  \begin{subfigure}[b]{0.49\columnwidth} 
    \centering
    \includegraphics[width=0.9\linewidth]{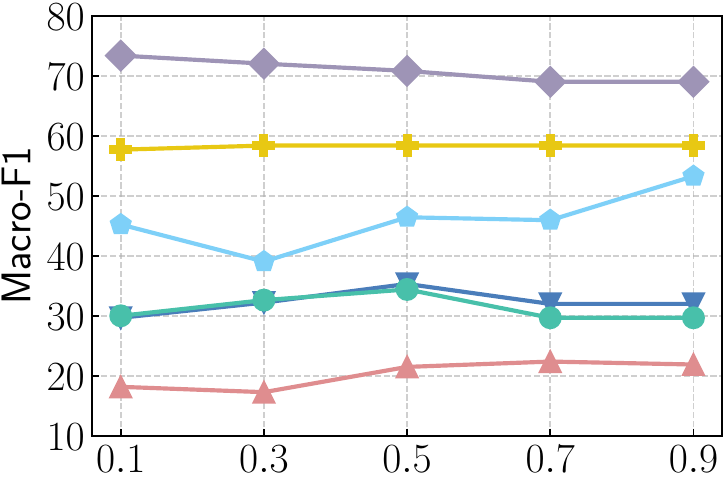}
    \caption{$l$ Macro-F1}
  \end{subfigure}

  \caption{\myblue{Parameter sensitivity analysis of $\gamma$ and $l$.}}
  \label{fig:hyperparams}
\end{figure}



\noindent

\subsection{Parameter Analysis}
\label{exp_params}
\myblue{
In this section, we further study the effect of varying the key parameters in $\text{C}^{2}\text{TC}$, including the adaptive weighting exponent $\gamma$ and the step size scaling factor $l$.
We conduct the evaluation on AD, JA, DA, AI, CO, and MI datasets.

Figure~\ref{fig:hyperparams}(a) and (b) illustrate the performance of $\text{C}^{2}\text{TC}$ for $\gamma \in \{0.25, 0.5, 0.75, 1.0\}$.
We observe that as $\gamma$ increases, accuracy generally decreases while Macro-F1 improves. This trade-off arises because smaller $\gamma$ values prioritize majority-class clustering, favoring overall accuracy, whereas larger $\gamma$ values emphasize minority-class clustering, leading to improved Macro-F1. 
Datasets with extreme imbalance ratios (e.g., JA, CO, and MI) exhibit distinct fluctuations, particularly in Macro-F1. This is because $\gamma$ effectively reweights the allocation of representatives toward minority classes; under highly imbalanced settings, even small changes in $\gamma$ can substantially affect minority-class representation. In practice, we suggest $\gamma=0.25$ as a default setting, since it already yields strong overall performance without overly sacrificing Macro-F1.
For scenarios with extreme class imbalance, we recommend setting $\gamma = 1$, as a larger $\gamma$ places greater emphasis on minority classes and improves minority-class classification performance.

Figure~\ref{fig:hyperparams}(c) and (d) show the performance of $\text{C}^{2}\text{TC}$ when varying $l \in \{0.1, 0.3, 0.5, 0.7, 0.9\}$. 
We observe that multi-class datasets (JA, DA, CO, MI) are more sensitive to changes than binary ones (AD, AI). Specifically, larger-scale datasets such as CO and MI tend to benefit from larger values of $l$, while JA and DA peak around $l=0.5$. This behavior can be attributed to the balance between exploration and exploitation: 
smaller $l$ values accelerate step-size reduction, enabling finer adjustments in later iterations, while larger $l$ values preserve broader exploration. 
In practice, we suggest $l=0.5$ as a default value, as it offers a better balance between sufficient exploration in early iterations and stable exploitation in later stages. For large-scale multi-class datasets, moderately increasing the step size ($l$) can further enhance performance by preventing premature convergence.
}

\section{Conclusion}
\label{sec:conclustion}
We propose $\text{C}^{2}\text{TC}$, the first training-free and label-adaptive tabular data condensation framework for efficient deep tabular learning.
To address the efficiency bottleneck of existing condensation methods, we reformulate the gradient-based strategies, which are typically entangled with complex optimization objectives, into a unified, training-free partition problem.
Building upon this formulation, we further extend it into a class-adaptive cluster allocation problem to support dynamic label allocation and enhance condensation flexibility.
To solve this NP-hard problem, we design a first-improvement local search algorithm guided by the elbow heuristic, enabling fast convergence to high-quality feasible solutions.
In addition, a hybrid categorical feature encoding technique is introduced to process heterogeneous discrete attributes while preserving their semantic information during clustering. 
Finally, extensive experiments on 10 real-world benchmark datasets demonstrate that $\text{C}^{2}\text{TC}$ achieves superior downstream performance, indicating higher data utility, while substantially reducing computational overhead over state-of-the-art baselines. 

\balance
\section*{AI-Generated Content Acknowledgement}
The authors used ChatGPT solely for proofreading and minor linguistic refinement.
No part of the technical content, experimental design, or analysis was generated by ChatGPT or other AI tools.

\bibliography{sample}

\clearpage       
\appendices

\label{appendix}
\section{Theoretical Proofs}
\subsection{Proof of Theorem~\ref{thm:gm-mtt}}
\begin{proof}
We consider standard gradient updates:
\begin{equation}
\scriptsize
\begin{split}
  \mathbf{W}^{\mathcal{T}}_{t+1} & = \mathbf{W}^{\mathcal{T}}_t - \eta \nabla \mathcal{L}^{\mathcal{T}}(\mathbf{W}^{\mathcal{T}}_t), \\
  \mathbf{W}^{\mathcal{S}}_{t+1} & = \mathbf{W}^{\mathcal{S}}_t - \eta \nabla \mathcal{L}^{\mathcal{S}}(\mathbf{W}^{\mathcal{S}}_t). 
\end{split}
\end{equation}
Then, the difference between the two trajectories at step $t+1$ can be calculated as:
\begin{equation}
\scriptsize
\begin{split}
\mathbf{W}^{\mathcal{S}}_{t+1} - \mathbf{W}^{\mathcal{T}}_{t+1} 
&= \left(\mathbf{W}^{\mathcal{S}}_t - \eta \nabla \mathcal{L}^{\mathcal{S}}(\mathbf{W}^{\mathcal{S}}_t)\right) - \left(\mathbf{W}^{\mathcal{T}}_t - \eta \nabla \mathcal{L}^{\mathcal{T}}(\mathbf{W}^{\mathcal{T}}_t)\right) \\
&= \left(\mathbf{W}^{\mathcal{S}}_t - \mathbf{W}^{\mathcal{T}}_t\right) - \eta \left(\nabla \mathcal{L}^{\mathcal{S}}(\mathbf{W}^{\mathcal{S}}_t) - \nabla \mathcal{L}^{\mathcal{T}}(\mathbf{W}^{\mathcal{T}}_t)\right).
\end{split}
\end{equation}
Taking the norm, we have:
\begin{equation}
\scriptsize
\label{norm}
\left\| \mathbf{W}^{\mathcal{S}}_{t+1} - \mathbf{W}^{\mathcal{T}}_{k+1} \right\| \leq \left\| \mathbf{W}^{\mathcal{S}}_t - \mathbf{W}^{\mathcal{T}}_t \right\| + \eta \left\| \nabla \mathcal{L}^{\mathcal{S}}(\mathbf{W}^{\mathcal{S}}_t) - \nabla \mathcal{L}^{\mathcal{T}}(\mathbf{W}^{\mathcal{T}}_t) \right\|.
\end{equation}
Finally, we apply Eq.~(\ref{norm}) recursively from $t=0$ to $t=T-1$, noting that $\mathbf{W}_{0}^{\mathcal{S}}=\mathbf{W}_{0}^{\mathcal{T}}$, we obtain:
\begin{equation}
\scriptsize
\left\| \mathbf{W}^{\mathcal{S}}_T - \mathbf{W}^{\mathcal{T}}_T \right\|  \leq \eta \sum_{k=0}^{T-1} \left\| \nabla \mathcal{L}^{\mathcal{S}}(\mathbf{W}^{\mathcal{S}}_t) - \nabla \mathcal{L}^{\mathcal{T}}(\mathbf{W}^{\mathcal{T}}_t) \right\|.
\end{equation}
\end{proof}

\subsection{Proof of Theorem~\ref{thm:dm-gm}}
\begin{proof}
As the gradient matching objective is commonly calculated for each class separately, we define the 
class-wise loss for $\mathcal{T}$ and $\mathcal{S}$ as:
\begin{equation}
\resizebox{\linewidth}{!}{$
\begin{aligned}
\mathcal{L}_{\mathrm{GM}} 
&= \sum_{i=0}^{n-1} \Big\|
\frac{1}{|n_i|}\nabla \mathcal{L}_i^{\mathcal{T}}(\mathbf{W}) 
- \frac{1}{|n'_i|}\nabla \mathcal{L}_i^{\mathcal{S}}(\mathbf{W})
\Big\|^2 \\
&= \sum_{i=0}^{C-1} \Big\|
\frac{1}{|n_i|}(\Phi_i^\top \Phi_i \mathbf{W} - \Phi_i^\top Y_i)
- \frac{1}{|n'_i|}({\Phi'_i}^\top \Phi'_i \mathbf{W} - {\Phi'_i}^\top Y'_i)
\Big\|^2 \\
&= \sum_{i=0}^{C-1} \Big\|
\Big(\tfrac{1}{|n_i|}\Phi_i^\top \Phi_i - \tfrac{1}{|n'_i|}{\Phi'_i}^\top \Phi'_i\Big)\mathbf{W}
- \Big(\tfrac{1}{|n_i|}\Phi_i^\top Y_i - \tfrac{1}{|n'_i|}{\Phi'_i}^\top Y'_i\Big)
\Big\|^2 \\
&\le \sum_{i=0}^{C-1} \Big\|
\tfrac{1}{|n_i|}\Phi_i^\top Y_i - \tfrac{1}{|n'_i|}{\Phi'_i}^\top Y'_i
\Big\|^2
+ \sum_{i=0}^{C-1} \Big\|
\tfrac{1}{|n_i|}\Phi_i^\top \Phi_i - \tfrac{1}{|n'_i|}{\Phi'_i}^\top \Phi'_i
\Big\|^2 \|\mathbf{W}\|^2 \\
&= \big\| \mathbf{P}\Phi(X) - \mathbf{P}'\Phi(X') \big\|^2
+ \sum_{i=0}^{C-1} \Big\|
\tfrac{1}{|n_i|}\Phi_i^\top \Phi_i - \tfrac{1}{|n'_i|}{\Phi'_i}^\top \Phi'_i
\Big\|^2 \|\mathbf{W}\|^2.
\end{aligned}
$}
\end{equation}

This decomposition reveals that distribution matching, when combined with second-order embedding alignment, provides an upper bound for class-wise gradient matching.
\end{proof}

\subsection{Proof of Theorem~\ref{thm:pgsa-nphard}}
\begin{proof}
Recall that MCKP is as follows: Given $m$ mutually disjointed groups, where group $G_{i}$ contains a set of $K_{i}$ items $\{x_{i,1},x_{i,2}, \dots, x_{i,K_{i}}\}$.
Each item $x_{i,k} \in G_{i}$ has a value $v_{i,k} \in \mathbb{R}_{\geq 0}$ and a weight $w_{i,k} \in \mathbb{Z}_{>0}$. The capacity of the knapsack is $W$.
MCKP aims to select exactly one item from each group such that the total weight $\sum_{i} w_{i,\sigma(i)} \leq W$ is within the budget and the total value $\sum_i v_{i,\sigma(i)}$ is maximized. 
Here, the mapping $\sigma: \{1,2,...,m\} \rightarrow \mathbb{Z}_{>0}$ specifies the index of the selected item in each group, i.e., $x_{i,\sigma(i)} \in G_{i}$ denotes the item chosen for group $i$. 
Based on this, we construct an instance of class-adaptive cluster allocation as: there are $m$ classes of samples which corresponds to the number of groups in MCKP.
The candidate set of cluster numbers for class $i$ is denoted as $\{n_{i,1}^{\prime},n_{i,2}^{\prime},\cdots,n_{i,K_{i}}^{\prime}\}$, where $n_{i,k}^{\prime} = k$ and $K_{i}=\max\left\{N^{\prime}-(C-1), \,n_{i}\right\}$. This defines the discrete search space for each class $i$, which is structurally equivalent
to the item set $G_{i}$ in MCKP. Each element in this set corresponds to an item in MCKP, where the sample count $n^{\prime}_{i,k}$ also plays a role of the item's weight $w_{i,k}$.
Assume that we enumerate all candidate $n^{\prime}_{i,k}$, and compute corresponding optimal clustering loss $\mathcal{L} (\{S_{j}^{i,k}\}^{*})$, which corresponds to the negative value $-v_{i,k}$ in MCKP. The total cluster budget is set to $W$, satisfying that $\sum_{i} w_{i,\sigma(i)} = W$. 
From this instance, we can derive a perfect matching for the MCKP problem. This completes the NP-hardness proof.
\end{proof}

\subsection{Proof of Theorem~\ref{thm:complexity}}
\begin{proof}
In Algorithm~\ref{alg:ifins}, the initialization stage incurs its main cost in the ClassWiseClustering (Algorithm~\ref{alg:ccap}), invoked at line~6 of Algorithm~\ref{alg:ifins}, where each class $i$ performs RunKMeans on $X_i \in \mathbb{R}^{n_i \times F}$ for $I$ iterations.
This step requires $\mathcal{O}(n_i n_i' I F)$ time, where $n_i' = \lfloor n_i \cdot r \rfloor$.
Hence, the overall initialization complexity is $\mathcal{O}(I F \sum_{i=0}^{C-1} n_i \lfloor n_i r \rfloor)$.
After initialization, the algorithm iteratively alternates between SoftAllocate (Algorithm~\ref{alg:softallocate}, invoked at line~11 of Algorithm~\ref{alg:ifins}), which runs in constant time, and ClassWiseClustering (line~12), which dominates the iteration cost.
In the worst case, clustering must be recomputed for all classes in each iteration.
Let $\bar{n}' = \frac{1}{C} \sum_{i=0}^{C-1} n_i' = \frac{N'}{C}$ denote the average number of clusters per class.
Then, over $T$ iterations, the total cost of the search phase is
$\mathcal{O}(T \sum_{i=0}^{C-1} n_i \bar{n}_i I F) = \mathcal{O}(T \cdot \frac{N N' I F}{C})$.
After convergence, Algorithm~\ref{alg:ifins} performs RunKMeans$(X_i, n_i^{\prime })$ for each class (line~22) to obtain the final partitions $\{\{S^{i}_{j}\}^{*}\}$.
This post-processing cost is negligible when $T$ is large.
The time complexity of the search over $T$ steps is $\mathcal{O}(T \cdot \frac{NN^{\prime}IF}{C})$. Overall, the time complexity of HFILS is $\mathcal{O}(IF\sum_{i=0}^{C-1}n_{i} \cdot \left\lfloor n_{i} \cdot r \right\rfloor + T \cdot \frac{NN^{\prime}IF}{C})$ = $\mathcal{O}(T \cdot \frac{NN^{\prime}IF}{C})$.
\end{proof}

\section{Related Work}
\subsection{Coreset selection}
\label{cs_rw}
Coreset selection methods aim to select a representative subset that achieves comparable model performance to training on the full dataset. 
Existing methods primarily rely on intermediate objectives that enforce distributional consistency between the subset and the original dataset~\cite{welling2009herding,feldman2011scalable,bachem2015coresets,farahani2009facility,cohen2022improved}.  
Representative approaches include herding~\cite{welling2009herding}, which greedily adds samples to minimize the feature-space distance between the centers of the subset and the full dataset, and $k$-center~\cite{farahani2009facility}, which iteratively selects samples to maximize spatial coverage.  
More recently, some methods have leveraged training signals such as gradients, loss values, or forgetting events to guide the selection of informative samples~\cite{mirzasoleiman2020coresets,killamsetty2021grad,toneva2018an,shrivastava2016training,hadar2024datamap}.  
Despite their effectiveness, coreset selection methods share inherent limitations: since they can only select a small subset of existing samples, they are constrained by the original dataset and cannot generate new, potentially more representative instances~\cite{zhao2023dataset,cazenavette2022dataset,zhao2021dataset}.  
Therefore, they fail to fully exploit the rich information in the full dataset, inevitably leading to substantial information loss~\cite{killamsetty2021grad}.  
This problem is further amplified when informative signals are broadly distributed rather than concentrated in a few samples~\cite{mirzasoleiman2020coresets}.

\section{Comparison with CGC}
\label{cgc}
\subsection{Difference in problem definition}
\myblue{The underlying problem definition and 
\red{structural assumptions} in our work differ fundamentally from those in~\cite{gao2025rethinking}.
Specifically,~\cite{gao2025rethinking} focuses on graph data condensation, aiming to distill a large-scale graph into a compact yet informative one while preserving training utility for deep graph models.
Accordingly, structural dependencies among entities are explicitly modeled by the adjacency matrix $A$ and retained in the condensed graph.
In contrast, our formulation targets tabular data composed of independent and identically distributed samples with heterogeneous feature types.
We explicitly distinguish numerical and categorical attributes, denoted as $x_{i,j}^{num}$ and $x_{i,j}^{cat}$, respectively, and aim to condense feature-heterogeneous tabular samples for efficient training of deep tabular models without modeling any inter-sample structural information.}
\subsection{Difference in theoretical framework}
\red{Our theoretical analysis process for establishing a unified condensation framework differs from~\cite{gao2025rethinking}}.
\cite{gao2025rethinking} unifies performance matching, parameter matching, and distribution matching in graph condensation, treating gradient matching as a representative instance of parameter matching and establishing DM as a proxy for PM. In contrast, we focus on independent-sample condensation and provide a finer-grained analysis of parameter matching by explicitly distinguishing gradient matching and training trajectory matching. We show that step-wise gradient discrepancies upper-bound long-range trajectory misalignment and further prove that DM upper-bounds GM, thereby yielding a unified framework encompassing MTT, GM, and DM.

\subsection{Difference in label allocation strategies}
\red{Our class partitioning objective and optimization space differ from~\cite{gao2025rethinking}.}
\cite{gao2025rethinking} focuses on graph condensation and performs partitioning on structure-aware node representations obtained via linear message passing. In contrast, we target tabular data condensation and, under a linearization simplification, conduct partitioning directly in the original sample space, since no structural dependencies exist among samples. Moreover, \cite{gao2025rethinking} considers only within-class partitioning with a fixed number of condensed samples per class, whereas we jointly optimize inter-class label allocation and intra-class clustering, leading to the Class-Adaptive Cluster Allocation Problem (CCAP), which adaptively determines the number of condensed samples per class.

\red{\section{Autoencoder for String-Type Feature Encoding}
\subsection{Computational Cost Analysis}
To assess the computational efficiency of the autoencoder module}, we empirically evaluate its runtime on three representative tabular datasets with string-type categorical features, namely Adult (AD), Diabetes130US (DA), and Airlines (AI). The results show that the encoding process is highly efficient.
In practice, the model converges within a small number of epochs and is therefore trained for a fixed 10 epochs as an offline preprocessing step.
Under this setting, the runtime is 6.8 seconds on AD, 46.3 seconds on DA, and 98.5 seconds on AI.
Overall, the proposed autoencoder serves as a lightweight and efficient preprocessing component that is trained only once per dataset, incurring a low computational cost that is negligible compared to the training overhead of prior learning-based condensation methods such as GM and DM.

\section{More Experimental settings}

\subsection{Detailed Datasets Descriptions}
\label{dataset_app}

To evaluate the effectiveness and efficiency of our approach, we conduct experiments on 10 publicly available tabular datasets from the widely adopted OpenML benchmarks~\cite{vanschoren2014openml}. 
These datasets cover diverse real-world domains, including but not limited to energy forecasting, income prediction, and flight delay analysis, and encompass heterogeneous data types, such as numerical and categorical features.
Their sizes range from the order of $10^{4}$ to $10^{6}$ instances, ensuring a comprehensive evaluation across small- to large-scale settings.
In addition, we include datasets with different label distributions, reflecting varying degrees of class imbalance commonly observed in real-world data. 
To measure this imbalance, we adopt the widely used metric known as the ratio of minimum to maximum class frequency~\cite{mcelfresh2023when}, where smaller values indicate a higher degree of imbalance. 
For all datasets, we randomly split them for 80\%/10\%/10\% as train/validation/test sets.
Table~\ref{tab:datasets} provides the detailed statistics of all datasets.

\subsection{Detailed Baselines Descriptions}
\label{comp_method_app}
We compare our method against a set of representative baselines, including coreset selection approaches (Random, Herding~\cite{welling2009herding}, and K-Center~\cite{farahani2009facility}) and dataset condensation methods (DM~\cite{zhao2023dataset}, GM~\cite{zhao2021dataset}, and MTT~\cite{cazenavette2022dataset}).
\begin{itemize}[leftmargin=*]
    \item \textbf{Random.} Random selects a subset of training samples uniformly at random from the original data as the coreset.

    \item \textbf{Herding}~\cite{welling2009herding}. Herding is a greedy algorithm that iteratively selects samples closest to the cluster center.

    \item \textbf{K-Center}~\cite{farahani2009facility}. K-Center picks multiple center points such that the largest distance between a data point and its nearest center point is minimized.

    \item \textbf{DM}~\cite{zhao2023dataset}. DM learns condensed dataset by matching the feature embeddings between real and synthetic samples. 
    
    \item \textbf{GM}~\cite{zhao2021dataset}. GM proposes to infer the synthetic dataset by matching the gradients of a model trained on synthetic dataset to that on the original
dataset.

    \item \textbf{MTT}~\cite{cazenavette2022dataset}. MTT builds a condensation dataset to match the parameter trajectory of full dataset training.
\end{itemize}
\subsection{Evaluation Metrics Details}
\label{eval_app}
To evaluate the classification performance of models trained on condensed datasets, we adopt two widely used metrics: \emph{Accuracy} and \emph{Macro-F1}.
Accuracy is defined as the proportion of correctly predicted instances:
\begin{equation}
\text{Accuracy} = \frac{TP + TN}{TP + TN + FP + FN},
\end{equation}
where $TP$, $TN$, $FP$, and $FN$ denote true positives, true negatives, false positives, and false negatives, respectively.
To further provide a balanced view of  classification performance across all classes,
we utilize Macro-F1, which is the harmonic mean of precision and recall, computed independently for each class and then averaged: 
\begin{equation}
    \text{Macro-F1} = \frac{1}{C} \sum_{c=0}^{C-1} \frac{2 \cdot \text{Precision}_c \cdot \text{Recall}_c}{\text{Precision}_c + \text{Recall}_c},
\end{equation}
where $\text{Precision}_c = \frac{TP_c}{TP_c + FP_c}$ and $\text{Recall}_c = \frac{TP_c}{TP_c + FN_c}$.
\subsection{Implementation Details}
\label{implement_app}
The pipeline of dataset condensation has two stages: (1) \textbf{Condensation step}, where we apply the DC methods on the training set to obtain small-scale synthetic data. (2) \textbf{Evaluation step}, where we train deep models on the condensed data from scratch and then evaluate their 
performance on the test set of the original data. For existing DC methods (i.e., DM, GM, and MTT) that requires model training, we use FT-Transformer~\cite{gorishniy2021revisiting}, an effective and widely adopted tabular learning method, as the relay model.  
Our effectiveness experiments consider three reduction rates (1\%, 0.1\%, and 0.01\%), with 0.1\% used as the default setting in subsequent experiments unless stated otherwise.
For $\text{C}^{2}\text{TC}$, we set the maximum iteration to $T=1000$, the tolerance to $\epsilon=0.01$, and the early-stopping patience to $p=10$.
The adaptive weighting exponent $\gamma$ is tuned over \{0.25, 0.5, 0.75, 1.0\}. Besides, we tune the step decay $l$ in \{0.1, 0.3, 0.5, 0.7, 0.9\}.
For all baselines, we report results using the official implementations with default hyperparameters specified in the original papers, and apply label encoding for categorical features. Methods that exceed the runtime limit (24 hours during condensation, or 5 days for expert trajectories generation during pre-processing in MTT) are denoted as OOT (out of time), while those that fail due to memory allocation errors are denoted as OOM (out of memory).
The experiments are conducted on a  Linux server with an Intel(R) Xeon(R) Gold 6330 CPU, a NVIDIA A800 GPU, and 128GB of RAM. 
To eliminate randomness, each experiment is repeated five times, and the average results with standard deviation are reported.

\section{Additional Experimental Results}

\subsection{Additional Performance Analysis}
\myblue{We observe that on MI, $\text{C}^{2}\text{TC}$ already recovers over 93\% accuracy of the full-data performance at $r=0.01\%$, with further increasing $r$ to $1\%$ yielding only marginal gains (e.g., from 53.2\% to 54.4\%). This behavior can be attributed to the following reasons. 
First, benefiting from class-adaptive clustering, $\text{C}^{2}\text{TC}$ is able to extract highly representative samples even at extremely low reduction rates. Second, the large-scale MI dataset exhibits substantial redundancy, allowing most informative patterns to be captured at very low reduction rates and leaving limited room for further improvement. 
Moreover, although increasing $r$ from 0.01\% to 1\% corresponds to a 100$\times$ relative increase, the absolute data scale remains small, resulting in only a modest performance
gain. 
Based on these observations, we recommend small $r$ values (e.g., $r=0.01\%$) when computational or memory budgets are limited, and $r=1\%$ as a robust default when a better balance between performance and generalization is desired.}

\myblue{ To assess the reliability of the observed improvements in accuracy (Acc) and Macro-F1 (MF1), we conduct paired t-tests~\cite{zar1999biostatistical} comparing $\text{C}^{2}\text{TC}$ with the strongest coreset and condensation baselines, namely Herding and MTT, at $r=0.1\%$, and report the corresponding p-values.
Smaller p-values indicate stronger evidence that the observed performance differences are not due to random variation, and results with $p<0.05$ are considered statistically significant. The tests were performed over 25 paired experimental runs per dataset, obtained from 5 independently generated synthetic datasets, each under 5 random seeds. As shown in Table~\ref{tab:ttest_multi}, $\text{C}^{2}\text{TC}$ achieves statistically significant gains ($p<0.05$) in accuracy on 9 out of the 10 datasets against Herding and 8 out of the 9 available datasets against MTT (excluding EP due to OOT), as well as in Macro-F1 on 8 out of the 10 and 8 out of the 9 datasets, respectively. In the few cases where improvements are not statistically significant (e.g., on the JA dataset against Herding as measured by accuracy), our method remains statistically comparable to the coreset baseline while still significantly outperforming the condensation competitor (MTT, $p=3.75\times10^{-3}$).}
\begin{table}[h]
\caption{\myblue{Paired t-test results ($p$-values) comparing $\text{C}^{2}\text{TC}$ against competitive baselines (Herding and MTT). Bold values indicate statistical significance ($p < 0.05$). Results for MTT on EP are omitted due to OOT.}}
\label{tab:ttest_multi}
\centering
\scriptsize
\renewcommand{\arraystretch}{1.12}
\begin{tabular}{c|cc|cc}
\toprule
\multirow{2}{*}{\textbf{Dataset}}
& \multicolumn{2}{c}{\textbf{C$^2$TC vs. Herding}}
& \multicolumn{2}{c}{\textbf{C$^2$TC vs. MTT}} \\
\cmidrule(lr){2-3} \cmidrule(lr){4-5}
& \textbf{Acc $p$-val} & \textbf{MF1 $p$-val}
& \textbf{Acc $p$-val} & \textbf{MF1 $p$-val} \\
\midrule
EL & \textbf{6.90$\times$10$^{-33}$} & \textbf{7.34$\times$10$^{-30}$} & \textbf{3.78$\times$10$^{-15}$} & \textbf{8.79$\times$10$^{-15}$} \\
AD & \textbf{4.01$\times$10$^{-17}$} & \textbf{3.20$\times$10$^{-20}$} & \textbf{5.37$\times$10$^{-13}$} & \textbf{5.13$\times$10$^{-10}$} \\
JA & 4.76$\times$10$^{-1}$            & \textbf{1.44$\times$10$^{-6}$}  & \textbf{3.75$\times$10$^{-3}$}  & \textbf{2.48$\times$10$^{-2}$}  \\
DA & \textbf{4.45$\times$10$^{-9}$}   & \textbf{1.05$\times$10$^{-13}$} & \textbf{9.07$\times$10$^{-8}$}  & \textbf{1.78$\times$10$^{-12}$} \\
RS & \textbf{2.28$\times$10$^{-2}$}   & 5.56$\times$10$^{-2}$            & \textbf{2.18$\times$10$^{-2}$}  & \textbf{3.60$\times$10$^{-2}$}  \\
EP & \textbf{1.31$\times$10$^{-21}$}  & \textbf{1.86$\times$10$^{-19}$} & --                              & --                              \\
AI & \textbf{1.47$\times$10$^{-32}$}  & \textbf{8.37$\times$10$^{-26}$} & \textbf{1.99$\times$10$^{-22}$} & \textbf{3.74$\times$10$^{-14}$} \\
CO & \textbf{3.07$\times$10$^{-25}$}  & \textbf{4.79$\times$10$^{-20}$} & \textbf{3.39$\times$10$^{-13}$} & \textbf{8.42$\times$10$^{-15}$} \\
HI & \textbf{5.18$\times$10$^{-17}$}  & \textbf{2.11$\times$10$^{-22}$} & \textbf{6.82$\times$10$^{-10}$} & \textbf{2.55$\times$10$^{-7}$}  \\
MI & \textbf{2.95$\times$10$^{-3}$}   & 8.84$\times$10$^{-2}$            & 3.25$\times$10$^{-1}$            & 5.62$\times$10$^{-2}$            \\

\bottomrule
\end{tabular}
\end{table}

\subsection {Additional Ablation Study}
\label{ab_app}
\myblue{
\noindent\textbf{Categorical feature encoding.}
We evaluate the effectiveness of the proposed hybrid categorical feature encoding (HCFE) by comparing it with three widely adopted encoding strategies~\cite{borisov2022deep}: (i) label encoding, (ii) one-hot encoding followed by PCA, where PCA is applied to match feature dimensionality for fair comparison, and (iii) target encoding. All encoding strategies are evaluated on three datasets with varying scales (AD, DA, and AI). As reported in Table~\ref{tab:encoding}, HCFE consistently achieves the best performance across all datasets. For example, on the AD dataset, HCFE improves accuracy by 13.8\% and Macro-F1 by 7.9\% compared to Label Encoding, and also outperforms the most competitive baselines, including Target Encoding and One-hot + PCA, in both metrics. These results demonstrate that HCFE achieves superior accuracy and class-level balance, validating its effectiveness as an encoding strategy for condensed tabular data.
Furthermore, the successful execution of the condensation process with these diverse baselines indicates that our framework is encoder-agnostic and can accommodate alternative numerical representations.
}

\begin{table}[h]
\caption{\myblue{Comparison of categorical feature encoding strategies for data condensation.}}
\label{tab:encoding}
\centering
\scriptsize
\resizebox{\linewidth}{!}
{\begin{tabular}{c cc cc cc cc}
\toprule
\multirow{2}{*}{\textbf{Dataset}}
& \multicolumn{2}{c}{\textbf{Label}}
& \multicolumn{2}{c}{\textbf{One-hot + PCA}}
& \multicolumn{2}{c}{\textbf{Target}}
& \multicolumn{2}{c}{\textbf{HCFE}} \\
\cmidrule(lr){2-3}
\cmidrule(lr){4-5}
\cmidrule(lr){6-7}
\cmidrule(lr){8-9}
& Acc & MF1 & Acc & MF1 & Acc & MF1 & Acc & MF1 \\
\midrule
AD & 67.7 & 62.9 & 74.4 & 70.0 & 78.8 & 69.9 & \textbf{81.5} & \textbf{70.8} \\
DA & 50.2 & 24.3 & 51.3 & 30.8 & 47.1 & 33.5 & \textbf{51.8} & \textbf{34.4} \\
AI & 53.1 & 52.0 & 61.0 & 56.0 & 59.1 & 58.2 & \textbf{61.5} & \textbf{58.4} \\
\bottomrule
\end{tabular}}
\end{table}


\subsection{Impact of Categorical Feature Imbalance}
\label{sec:feature-imbalance}
To evaluate the robustness of our categorical feature encoding
method (HCFE) under severe feature imbalance commonly
observed in real-world tabular datasets, we conduct experiments
to examine the effects of varying imbalance levels on the
proposed encoding and downstream model performance across
three datasets with categorical features: AD, DA, and AI.
Specifically, for each dataset, we select a categorical feature and construct training sets with controlled imbalance ratios of 1:1, 10:1, 100:1, and 1000:1 between two category values. For datasets containing a binary categorical feature, we directly use that column (e.g., sex in AD and diabetesMed in DA).
For datasets without binary categorical features (e.g., AI), we select a multi-valued categorical column (DayOfWeek) and divide its seven unique values into two groups, assigning the first four days to 0 and the remaining three days to 1, yielding a binary feature on which controlled imbalance ratios can be imposed.
These imbalance levels are achieved via a combination of label-conditioned downsampling and controlled upsampling on the selected feature. To ensure a fair and controlled comparison, all imbalance settings use the same number of training samples, strictly preserve the original label distribution, and share an identical test set, thereby isolating the effect of feature imbalance from other confounding factors.
We evaluate downstream performance following the same condensation and learning pipeline using HCFE and the two baseline encoders. Figure~\ref{fig:feature_imbalance} reports both accuracy and Macro-F1, where the x-axis denotes the feature imbalance level (e.g., 1 corresponds to 1:1, 10 to 10:1). HCFE demonstrates consistently stable performance under different degrees of categorical feature imbalance on all three datasets. For example, on the AI dataset, the accuracy remains within a narrow range of 60.21\%–61.97\%, while Macro-F1 varies only between 58.76\% and 59.81\%. This limited performance fluctuation suggests that HCFE preserves stable representation quality even under extremely skewed categorical distributions.

\begin{figure}[h] 
  \centering
  \includegraphics[width=0.7\linewidth]{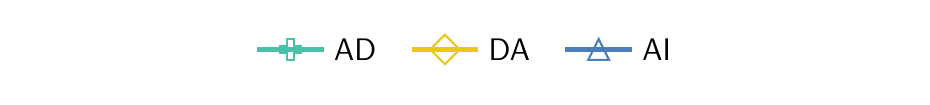}
  \begin{subfigure}[b]{0.47\columnwidth} 
    \centering
    \includegraphics[width=\linewidth]{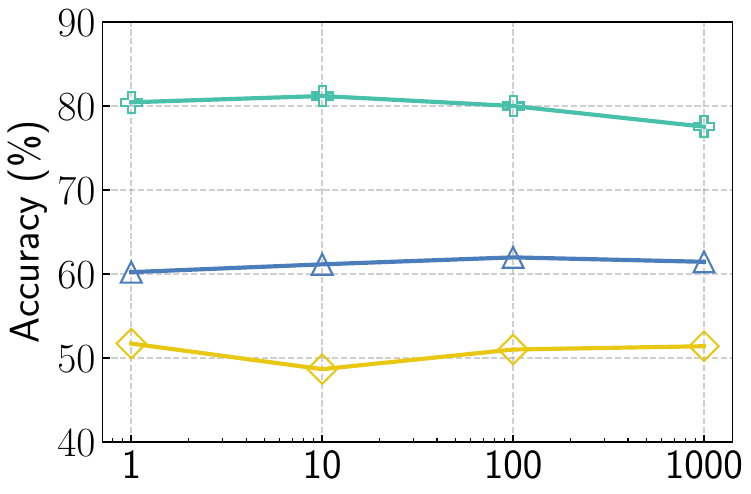}
    \caption{Accuracy}

  \end{subfigure}
    \begin{subfigure}[b]{0.47\columnwidth}
    \centering
    \includegraphics[width=\linewidth]{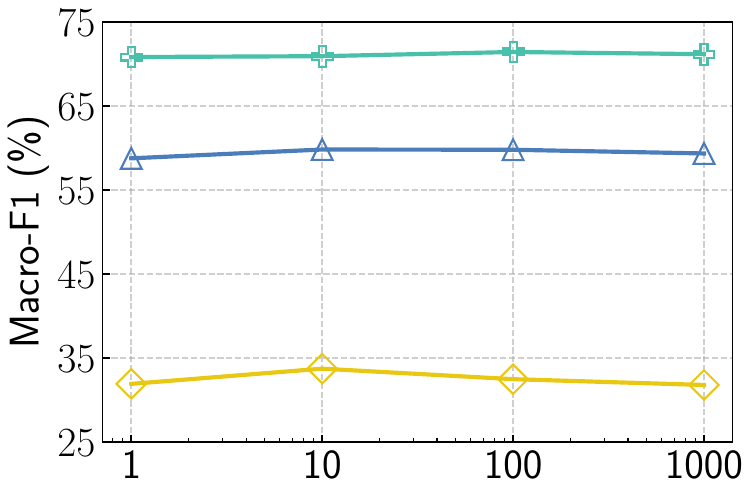}
    \caption{Macro-F1}

  \end{subfigure}
  \caption{\myblue{Impact of categorical feature imbalance.}}
  \label{fig:feature_imbalance}
\end{figure}

\clearpage 
\end{document}